\documentclass{article}

\usepackage{microtype}
\usepackage{graphicx}
\usepackage{booktabs}       
\usepackage{amsfonts}       
\usepackage{nicefrac}     
\usepackage{microtype}    
\usepackage{xcolor}        
\usepackage{graphicx}
\usepackage{subcaption}

\usepackage{xcolor}
\usepackage{soul}

\usepackage{enumitem}

\usepackage{balance}

\usepackage{hyperref}

\usepackage[accepted]{icml2024}

\usepackage{amsmath}
\usepackage{amssymb}
\usepackage{mathtools}
\usepackage{amsthm}

\usepackage[capitalize,noabbrev]{cleveref}

\theoremstyle{plain}
\newtheorem{theorem}{Theorem}[section]
\newtheorem{proposition}[theorem]{Proposition}
\newtheorem{lemma}[theorem]{Lemma}
\newtheorem{corollary}[theorem]{Corollary}
\theoremstyle{definition}
\newtheorem{definition}[theorem]{Definition}

\theoremstyle{remark}
\newtheorem{remark}[theorem]{Remark}

\DeclarePairedDelimiter{\ceil}{\lceil}{\rceil}

\DeclareMathOperator*{\argmax}{arg\,max} 
\DeclareMathOperator*{\argmin}{arg\,min} 

\newcommand{\bs}{\boldsymbol{s}}
\newcommand{\bo}{\boldsymbol{o}}
\newcommand{\bt}{\boldsymbol{t}}
\newcommand{\bx}{\boldsymbol{x}}
\newcommand{\by}{\boldsymbol{y}}

\icmltitlerunning{Stochastic $k$-Submodular Bandits with Full Bandit Feedback}

\begin{document}

\twocolumn[
\icmltitle{Stochastic $k$-Submodular Bandits with Full Bandit Feedback}

\icmlsetsymbol{equal}{*}

\begin{icmlauthorlist}
\icmlauthor{Guanyu Nie}{sch1}
\icmlauthor{Vaneet Aggarwal}{sch2}
\icmlauthor{Christopher John Quinn}{sch1}
\end{icmlauthorlist}

\icmlaffiliation{sch1}{Iowa State University (\texttt{nieg@iastate.edu}, \texttt{cjquinn@iastate.edu})}
\icmlaffiliation{sch2}{Purdue University (\texttt{vaneet@purdue.edu})}

\icmlcorrespondingauthor{Christopher John Quinn}{cjquinn@iastate.edu}

\icmlkeywords{Machine Learning, ICML}

\vskip 0.3in
]

\printAffiliationsAndNotice{}  

\begin{abstract}
In this paper, we present the first sublinear $\alpha$-regret bounds for online $k$-submodular optimization problems with full-bandit feedback, where $\alpha$ is a corresponding offline approximation ratio.   Specifically, we propose online algorithms for multiple $k$-submodular stochastic combinatorial multi-armed bandit problems, including (i) monotone functions and individual size constraints, (ii) monotone functions with matroid constraints, (iii)  non-monotone functions with matroid constraints, (iv) non-monotone functions without constraints, and (v) monotone functions without constraints. We transform approximation algorithms for offline $k$-submodular maximization problems into online algorithms through the offline-to-online framework proposed by \citet{nie23framework}. A key contribution of our work is analyzing the robustness of the offline algorithms.
\end{abstract}

\section{INTRODUCTION}

In various sequential decision problems, including sensor placement, influence maximization, and clinical trials, decisions are made by sequentially selecting a subset of elements, making assignments for those elements, and then observing the outcomes. These scenarios often exhibit diminishing returns based on the choice of the subset of elements and their corresponding assignments. 

Consider a multi-agent scenario where multiple companies (agents) cooperate to spread $k$ different types of content across a social network. Each company can select a subset of users and control how it distributes content to initiate propagation. The collective goal is to maximize the overall spread of all content types across the network. However, each company’s decisions impact not only its own success but also that of others. For instance, if two companies target users with overlapping follower networks, their efforts may lead to redundant influence, limiting the potential additional spread. Conversely, well-coordinated strategies could generate synergistic effects, amplifying the overall reach. This redundancy effect exists even for individual companies, making it crucial to target diverse users to avoid diminishing returns.

Each type of content follows an independent propagation process. Due to privacy restrictions, the companies only know the total number of users influenced by the content, without visibility into the underlying diffusion process. Additionally, because the diffusion is inherently random, the observations may vary, necessitating a balance between exploiting successful configurations (exploitation) and experimenting with different users and assignments to discover better strategies (exploration). See \cref{sec:supp:app} in the supplementary material for more motivating applications.

For each company, an offline variant of this sequential decision problem (i.e., where the agent has access to an exact value oracle and is only evaluated on a final output) can be modeled as a (un)constrained $k$-submodular optimization problem \cite{Huber2012TowardsMK}. 
The class of $k$-submodular functions generalizes the class of submodular set functions, an important non-linear function class for numerous optimization problems exhibiting a notion of diminishing returns, by not only accounting for the value of a subset of elements (e.g., which users) but also the assignment (e.g., which types of content to which users) as well. 
Maximizing a $k$-submodular function is known to be NP-hard even for unconstrained problems \cite{Ward2014MaximizingKF}. There has been significant progress in developing approximation algorithms for various offline $k$-submodular maximization problems \cite{Iwata2015ImprovedAA,ohsaka2015monotone,sakaue2017maximizing}. 

The sequential decision problem described above can be modeled as a stochastic combinatorial Multi-Armed Bandit (CMAB) problem with 
(i) an expected reward that is $k$-submodular, 
(ii) an action space formed by the constraints (if any) on elements and their assignments, and (iii) bandit feedback.  
Multi-Armed Bandits is a classical framework for sequential decision-making under uncertainty. 
Combinatorial MAB is a specialized branch where each action is a combination of base arms, known as a ``super arm'', and the number of super-arms is prohibitively large to try each one.    

For CMAB problems with non-linear rewards, the nature of the feedback significantly affects the problem complexity. ``Semi-bandit'' feedback involves observing additional information (such as the value of individual arms) aside from joint reward of the selected super arm, which greatly simplifies learning. In contrast, ``bandit'' or ``full-bandit'' feedback observes only the (joint) reward for actions. Estimating arm values for non-linear rewards with bandit feedback requires deliberate sampling of sub-optimal actions (e.g., such as actions consisting of individual base arms when the function is monotone), which is not typically done in standard MAB methods, since by design they do not take actions identified (with some confidence) to be sub-optimal. 

In this paper, we address the problem of stochastic CMAB with $k$-submodular expected rewards, various constraints, and only bandit feedback. 

{\bf Our Contributions: }
We propose and analyze the first CMAB algorithms with sub-linear $\alpha$-regret for  $k$-submodular rewards using only full-bandit feedback.  For each of the following results, we achieve them in part through analyzing the robustness of respective offline algorithms.

\begin{itemize}
    \item 
    For non-monotone $k$-submodular rewards and no constraints, we propose a CMAB algorithm with  $\Tilde{\mathcal{O}}\left(n k^{\frac{1}{3}} T^\frac{2}{3}\right)$ $1/2$-regret.  

    \item 
    For monotone $k$-submodular rewards and no constraints, we propose a CMAB algorithm with $\Tilde{\mathcal{O}}\left(n k^{\frac{1}{3}} T^\frac{2}{3}\right)$ $k/(2k-1)$-regret.  
    
    \item 
    For monotone $k$-submodular rewards under individual size constraints, we propose a CMAB algorithm with  $\Tilde{\mathcal{O}}\left(n^{\frac{1}{3}} k^{\frac{1}{3}} B T^\frac{2}{3}\right)$ $1/3$-regret. 

    \item 
    For monotone $k$-submodular rewards under a matroid constraint, we propose a CMAB algorithm with  $\Tilde{\mathcal{O}}\left(n^{\frac{1}{3}} k^{\frac{1}{3}} M T^\frac{2}{3}\right)$ $1/2$-regret.     We specialize our CMAB algorithm  for monotone functions with a matroid constraints to the case of total size constraints. 
    
    \item 
    For non-monotone $k$-submodular rewards under a matroid constraint, we propose a CMAB algorithm with $\Tilde{\mathcal{O}}\left(n^{\frac{1}{3}} k^{\frac{1}{3}} M T^\frac{2}{3}\right)$ $1/3$-regret.
\end{itemize}

In the offline setting, it is important to highlight that many of the properties of $f$ no longer hold in the presence of noise. For instance, in the case of a monotone function $f$, the marginal gains of $\hat{f}$ (noisy version of $f$) may no longer be guaranteed to be positive. Similarly, for non-monotone objectives, the property of pairwise monotonicity (see \cref{sec:prob:state}) for $f$ might not remain valid for $\hat{f}$. As a result, many proof steps in the original paper(s) that introduced those offline algorithms do not  go through directly and, consequently, need novel analysis. To address these challenges, this study incorporates bounded error to establish properties akin to the original ones in their respective contexts.

Analyzing the robustness of the offline algorithms is important, even without transforming them into an online setting. For those algorithms that possess inherent robustness, we refrain from modifications and maintain the original offline algorithm to ensure an efficient offline-to-online transformation. However, substantial modifications may be necessary for the original offline algorithm to maintain robustness in the presence of noise (e.g., \cref{sec:uc-monotone}). 
In such cases, we aim to introduce minimal alterations to the offline algorithm. We note that these adjustments require careful design, which is a part of the novelty of this paper.

\begin{remark} 
    The regret guarantees we obtained in this work are all of order $\mathcal{O}(T^{2/3})$. It is unknown whether $O(\sqrt{T})$ expected cumulative $\alpha$-regret is possible even in the case $k=1$. The only known results with bandit feedback for $k=1$ are $O(T^{2/3})$ for submodular bandits \cite{streeter2008online,niazadeh2021online}. While some lower bounds of $O(T^{2/3})$ exist (e.g., in adversarial setup with cardinality constraint \cite{niazadeh2021online}, or stochastic setup with cardinality constraint where regret is defined as gap to greedy algorithm \cite{tajdini2023minimax}), they are for specialized setups even for $k=1$, and thus a general understanding of lower bounds in these setups is an open problem. 
\end{remark}

\begin{table*}[h]
    \caption{ Summary of offline $\alpha$-approximation algorithms for $k$-submodular maximization with our $\delta$-robustness analysis  and  $\alpha$-regret bounds for our proposed  algorithms for $k$-submodular CMAB with full-bandit feedback. $N$ is an upper bound on the query complexity of the offline algorithm. $B$ is the total budget. $M$ is the rank of the matriod. There are no prior sublinear $\alpha$-regret bounds for $k$-submodular CMAB with full-bandit feedback. * \citet{sun2022maximize}'s result for non-monotone $f$ with matroid constraints specializes to the first results for non-monotone $f$ with total size constraints. $\dagger$ Both the $1/2$ approximation and the $\delta$ we obtain by analyzing the robustness of \citet{ohsaka2015monotone}'s offline approximation algorithm for monotone $f$ with total size constraints are strictly generalized by \citet{sakaue2017maximizing}'s offline approximation algorithm for monotone $f$ with matroid constraints.    
    } 
    \begin{center}
    \begin{tabular}{cccc|ccc}
    \toprule
        \textbf{Ref.} & 
        \textbf{Objective $f$} &
        \textbf{Constraint} & \textbf{$\alpha$} & \textbf{$\delta$} & $N$ & \textbf{Our $\alpha$-regret} \\
        \midrule
        \citet{Iwata2015ImprovedAA} & Non-Monotone & Unconstrained & $1/2$ & $20n$ & $nk$ & $\Tilde{\mathcal{O}}\left(n k^{\frac{1}{3}} T^\frac{2}{3}\right)$ \\
        \citet{Iwata2015ImprovedAA} & Monotone & Unconstrained & $k/(2k-1)$ & $(16-\frac{2}{k})n$ & $nk$ & $\Tilde{\mathcal{O}}\left(n k^{\frac{1}{3}} T^\frac{2}{3}\right)$ \\
        \citet{ohsaka2015monotone}$\dagger$ & Monotone & Total Size & $1/2$ & $B+1$ & $nkB$ & $\Tilde{\mathcal{O}}\left(n^{\frac{1}{3}} k^{\frac{1}{3}} B T^\frac{2}{3}\right)$\\
        \citet{sun2022maximize}* & Non-Monotone & Total Size & $1/3$ & $4/3(B+1)$ & $nkB$ & $\Tilde{\mathcal{O}}\left(n^{\frac{1}{3}} k^{\frac{1}{3}}B T^\frac{2}{3}\right)$ \\
        \citet{ohsaka2015monotone} & Monotone & Individual Size & $1/3$ & $4/3(B+1)$ & $nkB$ & $\Tilde{\mathcal{O}}\left(n^{\frac{1}{3}} k^{\frac{1}{3}} B T^\frac{2}{3}\right)$\\
        \citet{sakaue2017maximizing}$\dagger$ & Monotone & Matroid & $1/2$ & $M+1$ & $nkM$ & $\Tilde{\mathcal{O}}\left(n^{\frac{1}{3}} k^{\frac{1}{3}} M T^\frac{2}{3}\right)$\\
        \citet{sun2022maximize} & Non-Monotone & Matroid & $1/3$ & $4/3(M+1)$ & $nkM$ & $\Tilde{\mathcal{O}}\left(n^{\frac{1}{3}} k^{\frac{1}{3}} M T^\frac{2}{3}\right)$ \\
        \bottomrule
    \end{tabular}
    \end{center}
    \label{tab:summary}   
\end{table*}

{\bf Related works: } \label{sec:rw}
We briefly highlight closely related works. See \cref{sec:supp:rw} for an expanded discussion.
\paragraph{$k$-submodular CMAB} To the best of our knowledge, the only prior work for $k$-submodular CMAB for any constraint type and/or feedback model is presented in \cite{soma2019noregret}. They considered unconstrained $k$-submodular maximization under semi-bandit feedback in adversarial setting. For the non-monotone and monotone cases, they propose algorithms with $1/2$ and $\frac{k}{2k-1}$ regrets upper bounded by $\mathcal{O}(nk\sqrt{T})$ respectively. Their approach is based on Blackwell approachability \cite{Blackwell1956AnAO}. As discussed in the introduction, availability of semi-bandit feedback in CMAB problems with non-linear rewards significantly simplifies the learning problem (i.e., with semi-bandit feedback $\sqrt{T}$ $\alpha$-regret bound dependence is typically easy to achieve). We also note that the adversarial CMAB setting does not strictly generalize the stochastic setting. In the adversarial setting, the reward functions at each time step $\{f_t\}$ must exhibit $k$-submodularity but otherwise can vary widely. In the stochastic setting, the individual $f_t$'s are not necessarily $k$-submodular but the expected reward function, represented as $f = \mathbb{E}[f_t]$, is. 

\paragraph{Submodular CMAB}
For submodular rewards (i.e., $k=1$), \citet{streeter2008online} proposed and analyzed an algorithm for adversarial CMAB with submodular rewards, full-bandit feedback, and under a knapsack constraint (though only in expectation, taken over randomness in the algorithm). The authors  adapted a simpler greedy algorithm \cite{khuller1999budgeted}, using an $\epsilon$-greedy exploration type framework. \citet{niazadeh2021online} proposed a framework for transforming iterative greedy $\alpha$-approximation algorithms for offline problems to online methods in an adversarial bandit setting, for both semi-bandit (achieving $\widetilde{O}(T^{1/2})$ $\alpha$-regret) and full-bandit feedback (achieving $\widetilde{O}(T^{2/3})$ $\alpha$-regret). 
In the stochastic setting, \citet{nie2022explore} proposed an explore-then-commit type algorithm for online monotone submodular maximization under cardinality constraint. The result is extended with an  optimized stochastic-explore-then-commit approach in \citet{fourati2023combinatorial}. \citet{fourati2023randomized} proposed randomized greedy learning  algorithm for online non-monotone submodular maximization. \citet{nie23framework} recently proposed a general framework for adapting offline to online algorithms under full bandit feedback. Their framework require the adapted offline algorithm to satisfy the so called $(\alpha, \delta)$-robustness property. 

There are also a number of works that require additional ``semi-bandit'' feedback.  For combinatorial MAB with submodular rewards, a common type of semi-bandit feedback are marginal gains \cite{lin2015stochastic,yue2011linear, yu2016linear, takemori2020submodular}, which enable the learner to take actions of maximal cardinality or budget, receive a corresponding reward, and gain information not just on the set but individual elements.

\section{BACKGROUND} \label{sec:prob:state}

\subsection{$k$-Submodular Functions:} \label{sec:k-submod}

Let $k$ be a positive integer for the number of \textit{types} (i.e., types of stories) and $V=[n]$ be the ground set of \textit{elements} (i.e., users in a social network). Let $(k+1)^V:=\{(X_1,\ldots, X_k) | X_i\subseteq V, i\in \{1,\ldots,k\}, X_i\cap X_j=\emptyset, \forall i\neq j\}$. 
A function $f:(k+1)^V\rightarrow \mathbb{R}$ is called \textit{$k$-submodular} if, for any $\bx=(X_1,\ldots,X_k)$ and $\by=(Y_1,\ldots, Y_k)$ in $(k+1)^V$, we have
$$f(\bx)+f(\by) \geq f(\bx\sqcup \by)+f(\bx\sqcap \by)$$
where
    $\bx\sqcap \by:=(X_1\cap Y_1,\ldots, X_k\cap Y_k) , \bx\sqcup \by:=(X_1\cup Y_1\setminus (\cup_{i\neq 1}X_i\cup Y_i), \ldots, X_k\cup Y_k\setminus (\cup_{i\neq k}X_i\cup Y_i)).$
Note that setting $k=1$ in these definitions recovers submodular functions, set intersection, and set union, respectively. We further define a relation $\preceq$ on $(k+1)^V$ so that, for $\bx=(X_1, \ldots, X_k)$ and $\by=(Y_1, \ldots, Y_k)$ in $(k+1)^V, \bx \preceq \by$ if $X_i \subseteq Y_i$ for every $i$ with $i \in[k]$. 
We also define the \textit{marginal gain} of assigning type $i\in[k]$ to element $e$  given a current solution $\bx$ (provided that  $e$ has not been assigned any type in $\bx$),
\begin{align*}
    \Delta_{e, i} f(\bx)&=f(X_1, \ldots, X_{i-1}, X_i \cup\{e\}, X_{i+1}, \ldots, X_k) \\
    &\qquad -f(X_1, \ldots, X_k)
\end{align*}

for $\bx \in(k+1)^V, e \notin \bigcup_{\ell \in[k]} X_{\ell}$, and $i \in[k]$. 

\begin{theorem}\cite{Ward2014MaximizingKF}
    A function $f : (k + 1)^V \rightarrow \mathbb{R}$ is $k$-submodular if and only if $f$ satisfies the following two conditions:\\
    \textbf{Orthant submodularity:} $\Delta_{e, i} f(\bx) \geq \Delta_{e, i} f(\by)$ for any $\bx, \by \in (k+1)^V$ with $\bx \preceq \by, e \notin \bigcup_{\ell \in[k]} Y_{\ell}$, and $i \in[k]$; \\
    \textbf{Pairwise monotonicity:} $\Delta_{e, i} f(\bx)+\Delta_{e, j} f(\bx) \geq 0$ for any $\bx \in(k+1)^V, e \notin \bigcup_{\ell \in[k]} X_{\ell}$, and $i, j \in[k]$ with $i \neq j$.
\end{theorem}
We define the \textit{support} of $\bx \in(k+1)^V$ as $\operatorname{supp}(\bx)=\{e \in V \mid \bx(e) \neq 0\}$ and define the support of type $i\in[k]$ as $\operatorname{supp}_i(\boldsymbol{x})=\{e \in V \mid \boldsymbol{x}(e)=i\}$. Further, a function $f:(k+1)^V\rightarrow \mathbb{R}$ is called \textit{monotone} if $\Delta_{e, i} f(\bx) \geq 0$ for any $\bx \in(k+1)^V, e \notin \bigcup_{\ell \in[k]} X_{\ell}$, and $i \in[k]$.

\paragraph{Matroids} For a finite set $E$ and $\mathcal{F} \subseteq 2^E$, we say a system $(E, \mathcal{F})$ is a matroid if the following hold:
\begin{itemize}[noitemsep,topsep=0pt]
    \item (M1) $\emptyset \in \mathcal{F}$,
    \item (M2) If $A \subseteq B \in \mathcal{F}$ then $A \in \mathcal{F}$,
    \item (M3) If $A, B \in \mathcal{F}$ and $|A|<|B|$ then there exists $e \in B \backslash A$ such that $A \cup\{e\} \in \mathcal{F}$.
\end{itemize}

The elements of $\mathcal{F}$ are called independent, and we say $A \in \mathcal{F}$ is maximal if no $B \in \mathcal{F}$ satisfies $A \subsetneq B$. A maximal independent set $B$ is called a \emph{basis} of the matroid. The rank function of a matroid $\mathcal{M}= (E, \mathcal{F})$ is defined as $r_{\mathcal{M}(A)}=\max\{|S|:S\subseteq A,S\in \mathcal{F}\}$. Under the matroid constraint with a matroid $\mathcal{M}= (E, \mathcal{F})$, a solution $\bx \in(k+1)^V$ is feasible if $\bx \in \mathcal{F}$. 

A simple but important matroid constraint $\mathcal{M}= (E, \mathcal{F})$ is the uniform matroid in which $\mathcal{F} = \{X \in V: |X| \leq B\}$ for a given $B$. It is equivalent to the Total Size (TS) constraint in the problem we consider.

\subsection{CMAB} 

In the CMAB framework, we consider sequential, combinatorial decision-making problems over a finite time horizon $T$. Let $\Omega$ denote the ground set of base arms and $n=|\Omega|$ denote the number of arms.  Let $D \subseteq 2^\Omega$ denote the subset of feasible actions, for which we presume membership can be efficiently evaluated. We will use the terminologies \emph{subset} and \emph{action} interchangeably throughout the paper.  

At each time step $t$, the learner selects a feasible action $A_t \in D$.  After the subset $A_t$ is selected, the learner receives a reward $f_t(A_t)$. We assume the reward $f_t$ is stochastic, bounded in $[0,1]$, and i.i.d. conditioned on the action $A_t$.  Define the expected reward function as $f(A) := \mathbb{E}[f_t(A)]$. The goal of the learner is to maximize the cumulative reward $\sum_{t=1}^Tf_t(A_t)$. To measure the performance of the algorithm, one common metric is to compare the learner to an agent with access to a value oracle for $f$. However, if optimizing $f$ over $D$ is NP-hard, such a comparison would not be meaningful unless the horizon is exponentially large in the problem parameters. Alternatively, if there exists a known approximation algorithm $\mathcal{A}$ with an approximation ratio $\alpha\in(0,1]$ for optimizing $f$ over $D$, it is more natural to evaluate the performance of a CMAB algorithm against what $\mathcal{A}$ could achieve. In this scenario, we consider the expected cumulative $\alpha$-regret $\mathcal{R}_{\alpha,T}$, which quantifies the difference between $\alpha$ times the cumulative reward of the optimal subset's expected value and the average received reward, (we write $\mathcal{R}_T$ when $\alpha$ is understood from context)
\begin{align}
    \mathbb{E}[\mathcal{R}_{T}] = \alpha Tf(\mathrm{OPT}) - \mathbb{E}\left[\sum_{t=1}^T f_t(A_t)\right],\label{eq:reg:exp1e}
\end{align} 
where OPT is the optimal solution, i.e., $\text{OPT}\in \argmax_{A \in D } f(A)$ and the expectations are with respect to both the rewards and actions (if random). 

\subsection{Problem Statement}

We consider the sequential decision making problem under the stochastic CMAB framework where the expected reward function $f$ is $k$-submodular. 
Each arm consists of a tuple (we call it an item-type pair): $(e,i)\in V\times [k]$, and a super arm is defined as $\bx\in (k+1)^V$, which is a combination of base arms. We consider full-bandit feedback, where after each time an action $A_t$ is selected, the learner can only observe $f_t(A_t)$ as reward. We aim to transform various offline algorithms in $k$-submodular optimization literature to online algorithms, thus we consider $\alpha$ regret as our performance metric and $\alpha$ is defined to be the approximation ratio of the corresponding offline algorithm.

\subsection{Offline-to-Online Framework} \label{bg}

\begin{algorithm}[H]
\caption{C-ETC \cite{nie23framework}}
\label{alg:cetc}
\begin{algorithmic}[1]
    \STATE {\bfseries Input:}  horizon $T$, set $\Omega$ of $n$ base arms, an offline $(\alpha,\delta)$-robust algorithm $\mathcal{A}$, and an upper-bound $N$ on the number of  queries $\mathcal{A}$ will make to the value oracle. 
    \vspace{2mm}
    \STATE  Initialize $m\gets \ceil*{\frac{\delta^{2/3}T^{2/3}\log(T)^{1/3}}{2N^{2/3}}}$. 
    \vspace{2mm}
    \STATE  // Exploration Phase //
    \WHILE{$\mathcal{A}$ queries the value of some  $A\subseteq \Omega$}
        \STATE For $m$ times, play action $A$.
        \STATE Calculate the empirical mean $\bar{f}$. 
        \STATE Return $\bar{f}$ to $\mathcal{A}$.
    \ENDWHILE
    \vspace{2mm}
    \STATE // Exploitation Phase //
    \FOR{\emph{remaining time}}
        \STATE Play action $S$ output by algorithm $\mathcal{A}$.
    \ENDFOR
\end{algorithmic}
\end{algorithm}

As we adopt the novel offline-to-online transformation framework proposed in \cite{nie23framework}, in this section, we briefly introduce the framework. In \cite{nie23framework}, they introduced a criterion for the robustness of an offline approximation algorithm. They showed that this property alone is sufficient to guarantee that the offline algorithm can be adapted to solve CMAB problems in the corresponding online setting with just bandit feedback and achieve sub-linear regret. More importantly, the CMAB adaptation will not rely on any special structure of the algorithm design, instead employing it as a black box. We restate the robustness definition in the following. This definition of robustness indicates that the algorithm $\mathcal{A}$ will maintain reasonably good performance when function evaluations have errors.

\begin{definition}[$(\alpha, \delta, N)$-Robust Approximation \cite{nie23framework}]\label{def:robust}
    An algorithm (possibly random) $\mathcal{A}$ is an $(\alpha, \delta, N)$-robust  approximation algorithm
    for the combinatorial optimization problem of maximizing a function $f:2^\Omega\to \mathbb{R}$ over a finite domain $D \subseteq 2^\Omega$ if its output $S^*$ using a value oracle for $\hat{f}$ satisfies the relation below with the optimal solution $\mathrm{OPT}$ under $f$, provided that for any $\epsilon >0$ that $|f(S)-\hat{f}(S)| \leq \epsilon$ for all $S\in D$,
    \begin{align*}
        \mathbb{E}[f(S^*)]\geq \alpha f(\mathrm{OPT})-\delta \epsilon,
    \end{align*}
    where $\Omega$ is the ground set, the expectation is over the randomness of the algorithm $\mathcal{A}$, and algorithm $\mathcal{A}$ uses at most $N$ value oracle queries. 
\end{definition}

Note that when we have access to the exact oracle ($\varepsilon = 0$), the definition will give us the same guarantee as the original offline algorithm; if the offline algorithm is exact, $\alpha=1$. Equipped with a robustness assurance, the authors introduced a stochastic CMAB algorithm named ``Combinatorial Explore-Then-Commit'' (C-ETC). See \cref{alg:cetc} for the pseudocode. C-ETC interfaces with an offline $(\alpha,\delta, N)$-robust algorithm denoted as $\mathcal{A}$. In the exploration phase, when $\mathcal{A}$ requests information from the value oracle pertaining to action $A$, C-ETC adopts a strategy of executing action $A$ multiple times, specifically $m$ times, where $m$ is an optimization parameter. Subsequently, C-ETC computes the empirical mean, represented as $\bar{f}$, of the rewards associated with action $A$ and communicates this computed value back to the offline algorithm $\mathcal{A}$. In the exploitation phase, C-ETC continuously deploys the solution $S$ generated by the $\mathcal{A}$ algorithm. 
They showed the following theorem:
\begin{theorem}\cite{nie23framework}
\footnote{\citet{nie23framework}'s results were for deterministic offline approximation algorithms.  See \url{https://arxiv.org/pdf/2301.13326.pdf} for an extension to randomized offline approximation algorithms.}
\label{thm:cetc}
The expected cumulative $\alpha$-regret of C-ETC using an $(\alpha,\delta, N)$-robust approximation algorithm $\mathcal{A}$ as a subroutine is at most $\mathcal{O}\left(\delta^\frac{2}{3}N^\frac{1}{3} T^\frac{2}{3}\log(T)^\frac{1}{3}\right)$ with $T\geq \max\left\{N, \frac{2\sqrt{2}N}{\delta}\right\}$.
\end{theorem} 
This approach allows for a CMAB procedure that does not require any specialized structural characteristics from $\mathcal{A}$. It necessitates no intricate construction beyond the execution of $\mathcal{A}$, provided the criteria of robustness (Definition~\ref{def:robust}) are met. 

\begin{remark}
    We note that this result can be extended to multiple agents, where the exploration ($m$ times play of action $A$) happens in a distributed manner over the agents as shown in \citet{fouratifederated}. 
\end{remark}

In the following sections, we transform various offline algorithms for $k$-submodular maximization problems under different constraints to online bandit setting by analyzing the robustness of the offline algorithms. In analyzing the robustness of offline algorithms, we assume the offline algorithm is evaluated with a surrogate function $\hat{f}$ with $|f(S)-\hat{f}(S)| \leq \epsilon$ for all $S\in D$, instead of the exact function $f$. We will highlight some parts of the proof for non-monotone unconstrained problem, and defer all missing proofs to the appendix.

\begin{remark}[Offline v.s. Online Algorithms]
We note the following major differences between offline algorithms and online algorithms. First, an offline algorithm assumes exact oracle access to the objective function $f$. However, in certain scenarios, the objective might not be predetermined. Nonetheless, if we engage in recurrent tasks while encountering objectives $f_1, f_2, \cdots, f_T$, potentially sampled from a distribution, there is a possibility of acquiring the ability to perform satisfactorily on average over time. Such algorithms are referred to as online algorithms. Second, offline algorithms are designed to optimize an objective $f$ in the sense of finding a ``final'' solution, but the quality of intermediate solutions does not matter. Those offline algorithms care about computational and sample complexity that depends on problem size, but there is no sense of a horizon explicitly. In contrast, the aim of online algorithms is not to output a single solution (there is ``simple regret'' but we consider more common cumulative regret) but the sum of achieved values, so intermediate solutions matter and the horizon plays an explicit role.    
\end{remark}

\section{Non-monotone Functions without Constraints} \label{sec:uc}
In this section, we consider the case where the expected objective function is non-monotone $k$-submodular and there is no constraint in selecting actions. We adopt the offline algorithm proposed in \citet{Iwata2015ImprovedAA}.

\subsection{Algorithm}

We first remark the following key fact for (monotone or non-monotone) $k$-submodular maximization without constraints (see \citet{iwata2013bisub} for the proof):

\begin{proposition}\label{prop:Iwata2015ImprovedAA}
    For any $k$-submodular function $f:(k+1)^V \rightarrow \mathbb{R}_{+}$ with $k\geq 2$, there exists a partition of $V$ that attains the maximum value of $f$.
\end{proposition}
This says that there is an optimal solution with all elements included (a combinatorial search is needed to determine their types). 
Proposition~\ref{prop:Iwata2015ImprovedAA} remains valid for non-monotone functions due to the special property of $k$-submodular functions (pairwise monotonicity).

Armed with this property, they proposed a randomized offline algorithm for maximizing a non-monotone $k$-submodular maximization without constraint. 
The algorithm is presented in \cref{alg:uc} in the appendix. 
This algorithm iterates through all $e\in V$ (in an arbitrary but fixed order) and selects the type $i\in [k]$ randomly according to some carefully designed probabilities $p_i$ for each $e\in V$. The authors showed that this algorithm achieves a $1/2$-approximation ratio for non-monotone $k$-submodular maximization without constraints.

\subsection{Robustness} \label{sec:uc:robust}
We show that  Algorithm~\ref{alg:uc} is $(\frac{1}{2},\delta, N)$ robust for a certain $\delta$ and $N$.  Due to that robustness we can use the framework proposed in \citet{nie23framework}, to transform Algorithm~\ref{alg:uc} into an online algorithm under bandit feedback achieving sublinear $1/2$-regret. 

Let $\bo$ be an optimal solution with $\operatorname{supp}(\bo)=V$ (by Proposition~\ref{prop:Iwata2015ImprovedAA} such a solution exists). 
Let $\bs$ be the output of the algorithm with $|\operatorname{supp}(\bs)|=n$ using surrogate function $\hat{f}$. 
We consider the $j$-th iteration of the algorithm, and let $e^{(j)}$ be the element of $V$ considered in the $j$-th iteration, $p_i^{(j)}$ be the probability that $i$-th type is chosen in the $j$-th iteration, and $\bs^{(j)}$ be the solution after the $i$-th iteration, where $\bs^{(0)}=\mathbf{0}$. Also for $0 \leq j \leq n$, let $\bo^{(j)}$ be the element in $(k+1)^V$ obtained from $\bo$ by replacing the coordinates on $\operatorname{supp}(\bs^{(j)})$ with those of $\bs^{(j)}$, and for $1 \leq j \leq n$ let $\bt^{(j-1)}$ be the element in $(k+1)^V$ obtained from $\bo^{(j)}$ by changing $\bo^{(j)}(e^{(j)})$ with 0.

For $i \in[k]$, let $y_i^{(j)}=\Delta_{e^{(j)}, i} f(\bs^{(j-1)})$, $\hat{y}_i^{(j)}=\Delta_{e^{(j)}, i} \hat{f}(\bs^{(j-1)})$ and let $a_i^{(j)}=\Delta_{e^{(j)}, i} f(\bt^{(j-1)})$,
$\hat{a}_i^{(j)}=\Delta_{e^{(j)}, i} \hat{f}(\bt^{(j-1)})$. 
Due to pairwise monotonicity, we have $y_i^{(j)}+y_{i'}^{(j)} \geq 0$ and  $a_i^{(j)}+a_{i'}^{(j)} \geq 0$ for all $i, i' \in[k]$ with $i \neq i'$, and thus, while the surrogate function $\hat{f}$ does not necessarily have pairwise monotonicity, we have 
\begin{equation}
    \hat{y}_i^{(j)}+\hat{y}_{i'}^{(j)} \geq -4 \varepsilon , 
    \hat{a}_i^{(j)}+\hat{a}_{i'}^{(j)} \geq -4\varepsilon .   \label{eq:uc:pair2:main}
\end{equation}
for all $i, i' \in[k]$ with $i \neq i'$ since both inequalities involve a sum of four function values of $\hat{f}$ which can be off by at most $\epsilon$ with respect to $f$. 
Also from $\boldsymbol{s}^{(j)} \preceq \boldsymbol{t}^{(j)}$ which holds by construction,  orthant submodularity implies $y_i^{(j)} \geq a_i^{(j)}$ for all $i \in[k]$ and thus,
\begin{align}
    \hat{y}_i^{(j)} \geq \hat{a}_i^{(j)} -4\varepsilon. \label{eq:uc:orthant:main}
\end{align}
for all $i \in[k]$. W.l.o.g., we assume $f(\boldsymbol{0})=0$.

Before showing the main result, we first show the following lemma, which is a generalization of Lemma 2.2 in \citet{Iwata2015ImprovedAA} in the presence of noise.

\begin{lemma} \label{lem:uc:lem}
    Let $c, d \in \mathbb{R}_{+}$. Conditioning on $\bs^{(j-1)}$, suppose that
    \begin{align}
        \sum_{i=1}^k\left(\hat{a}_{i^*}^{(j)}-\hat{a}_i^{(j)}\right) p_i^{(j)} \leq c\left(\sum_{i=1}^k \hat{y}_i^{(j)} p_i^{(j)}\right) + d\varepsilon \label{eq:uc:condition:main}
    \end{align}
    holds for each $j$ with $1 \leq j \leq n$, where $i^*=\bo(e^{(j)})$. Then $\mathbb{E}[f(\bs)] \geq \frac{1}{1+c} f(\bo) -(2c+d+4)n\varepsilon$.
\end{lemma}

We defer the proof of this lemma to \cref{sec:prf:uc:lem}. Then, we show the following proposition:
\begin{proposition}\label{thm:nmuc:robust}
    Algorithm~\ref{alg:uc} for maximizing a non-monotone $k$-submodular function is a $(\frac{1}{2}, 20n, nk)$-robust approximation algorithm.
\end{proposition}

To show Proposition~\ref{thm:nmuc:robust}, by Definition~\ref{def:robust}, we aim to show that the output $\bs$ of Algorithm~\ref{alg:uc} satisfies $\mathbb{E}[f(\bs)] \geq \frac{1}{2}f(\bo)-20n\varepsilon$ when the function evaluation is based on surrogate function $\hat{f}$. By Lemma~\ref{lem:uc:lem} it suffices to prove \eqref{eq:uc:condition:main} for every $j\in [n]$ for $c=1$ and $d=14$. For simplicity of the description, we shall omit the superscript $(j)$ if it is clear from the context. Due to limited space, here we only consider the case when $i^+ \geq 3$ since this case needs more effort technically. The proof for other cases is in supplementary material \cref{sec:prf:uc}.

\begin{proof}[Part of proof of Proposition~\ref{thm:nmuc:robust}]
    ($i^{+} \geq 3$) Our goal is to show
    \begin{align}
        \sum_{i=1}^k (\hat{y}_i+\hat{a}_i) p_i \geq \hat{a}_{i^*} - 14\varepsilon, \label{eq:uc:non-mono:main}
    \end{align} 
    which is equivalent to \eqref{eq:uc:condition:main} with $c=1$ and $d=14$. By the ordering of $\{\hat{y}_i\}_{i\in[k]}$ and \eqref{eq:uc:orthant:main}, for $i\leq i^*$ we have
    \begin{align}
        \hat{y}_i \geq \hat{y}_{i^*} \geq \hat{a}_{i^*}-4\varepsilon. \label{eq:uc:case3:1}
    \end{align}
    Denote $r \in \argmin_{i \in[k]} \hat{a}_i $. 
    
    \textbf{Case 1:} If $r=i^*$, we have $\sum_i \hat{a}_i p_i \geq \hat{a}_{i^*}(\sum_i p_i)=\hat{a}_{i^*}$. Since $\sum_i \hat{y}_i p_i \geq 0$ and since from Line 11 in \cref{alg:uc},   a positive  probability $p_i>0$ is only assigned to positive types $i$ with positive $\hat{y}_i>0$ values,
    \begin{align*}
        \sum_{i=1}^k (\hat{y}_i+\hat{a}_i) p_i 
        \geq 0+ \hat{a}_{i^*} 
        \geq  \hat{a}_{i^*} - 14\varepsilon.
    \end{align*}    
    Thus \eqref{eq:uc:non-mono:main} follows. 

    \textbf{Case 2:} If $r \neq i^*$ and $i^* \geq i^{+}$, since from Line 11 in \cref{alg:uc},   a positive  probability $p_i>0$ is only assigned to positive $\hat{y}$ values, we have that $\sum_i \hat{y}_i p_i=\sum_{i \leq i^{+}} \hat{y}_i p_i \geq \sum_{i \leq i^{+}} (\hat{a}_{i^*} -4\varepsilon)p_i=\hat{a}_{i^*}-4\varepsilon$ by \eqref{eq:uc:case3:1}. 
    When $\hat{a}_r\geq 0$, since $\hat{a}_i\geq \hat{a}_r$ for any $i\neq r$ by definition of $r$, we have $\sum_i \hat{a}_i p_i\geq 0$. 
    When $\hat{a}_r < 0$,
    \begin{align}
        \sum_i \hat{a}_i p_i 
        &\geq \sum_{i \neq r}(-\hat{a}_r -4\varepsilon)  p_i+\hat{a}_r p_r\tag{by \eqref{eq:uc:pair2:main}}\\
        &= \hat{a}_r \left(p_r-\sum_{i \neq r}p_i \right) - 4\varepsilon \sum_{i \neq r}p_i \nonumber\\
        &\geq 0 - 4\varepsilon, \nonumber 
    \end{align}
    where the last inequality follows from $\hat{a}_r<0$, $\sum_{i \neq r} p_i \geq p_r$ and $\sum_{i \neq r}p_i\leq 1$, and $\sum_{i \neq r} p_i \geq p_r$ follows from Line 11 in \cref{alg:uc} that the largest single probability assigned is $\frac{1}{2}$. 

    Thus,        
    \begin{align*}
        \sum_{i=1}^k (\hat{y}_i+\hat{a}_i) p_i
        &\geq (\hat{a}_{i^*} - 4\varepsilon) + (- 4\varepsilon)
        \geq  \hat{a}_{i^*} - 14\varepsilon.
    \end{align*} 
    Therefore \eqref{eq:uc:non-mono:main} holds.
    
    \textbf{Case 3:} If $r \neq i^*$ and $i^*<i^{+}$,  we have
    \begin{align}
        &\sum_{i=1}^k (\hat{y}_i+\hat{a}_i) p_i = \sum_{i=1}^k \hat{y}_ip_i + \sum_{i=1}^k \hat{a}_ip_i \nonumber\\
        & \geq \sum_{i \leq i^*} (\hat{a}_{i^*} -4\varepsilon) p_i+\sum_{i>i^*} (\hat{a}_i-4\varepsilon) p_i+\sum_{i=1}^k \hat{a}_ip_i  \tag{using \eqref{eq:uc:case3:1} and \eqref{eq:uc:orthant:main}}\\
        &=\sum_{i \leq i^*} \hat{a}_{i^*} p_i+\sum_{i>i^*} \hat{a}_i p_i+\sum_{i=1}^k \hat{a}_i p_i - 4\varepsilon \nonumber\\
        & =\left(\sum_{i<i^*} \hat{a}_{i^*} p_i+2 \hat{a}_{i^*} p_{i^*}\right) \nonumber\\
        &\qquad +\left(\sum_{i>i^*} \hat{a}_i p_i+\sum_{i \neq r, i^*} \hat{a}_i p_i+ \hat{a}_r p_r\right) - 4 \varepsilon. \label{eq:uc:case3:2}
    \end{align}
    The first term equals $\hat{a}_{i^*}$ by simple calculations. Hence it suffices to show that the second term of \eqref{eq:uc:case3:2} is greater than or equal to $-10\varepsilon$. 

    \textbf{Sub-case 3.a:}   
    First, we consider $\hat{a}_r \geq 0$. 
    By definition of $r$, for all $i\in[k]$, $\hat{a}_i \geq \hat{a}_r \geq 0$. The second term of \eqref{eq:uc:case3:2} can then be bounded 
    \begin{align}
        &\sum_{i>i^*} \hat{a}_i p_i+\sum_{i \neq r, i^*} \hat{a}_i p_i+ \hat{a}_r p_r \nonumber\\
        &\qquad \geq \hat{a}_r \left(p_r + \sum_{i>i^*} p_i + \sum_{i \neq r, i^*} p_i\right) \geq 0 \geq -10\varepsilon. \label{eq:uc:case3:5}
    \end{align}

    \textbf{Sub-case 3.b:} 
    Second we consider $\hat{a}_r<0$. Since $i^*<i^{+}$, we have
    \begin{align}
        &\sum_{i>i^*} p_i+\sum_{i \neq r, i^*} p_i = \sum_{i^*+1\leq i\leq i^+ -1} p_i + p_{i^+}+\sum_{i \neq r, i^*} p_i \nonumber\\
        &=\sum_{i=i^*+1}^{i^+-1} \left(\frac{1}{2}\right)^{i} + \left(\frac{1}{2}\right)^{i^+-1} +\sum_{i \neq r, i^*} p_i \tag{by construction of $p_i$'s}\\
        &=\left(\frac{1}{2}\right)^{i^*}+\sum_{i \neq r, i^*} p_i \tag{geometric sum}\\
        &=p_{i^*}+\sum_{i \neq r, i^*} p_i =1-p_r. \label{eq:uc:case3:3}
    \end{align}

    Therefore, if $r<i^*$, we get
    \begin{align}
        &\sum_{i>i^*} \hat{a}_i p_i+\sum_{i \neq r, i^*} \hat{a}_i p_i+ \hat{a}_r p_r \nonumber\\
        &\geq \sum_{i>i^*} (-\hat{a}_r -4\varepsilon) p_i+\sum_{i \neq r, i^*} (-\hat{a}_r -4\varepsilon) p_i+ \hat{a}_r p_r \tag{using \eqref{eq:uc:pair2:main} }\\
        &= \hat{a}_r\left(p_r - \sum_{i>i^*} p_i - \sum_{i \neq r, i^*} p_i\right) - 4\varepsilon \left(\sum_{i>i^*} p_i + \sum_{i \neq r, i^*} p_i\right) \nonumber\\
        &=\hat{a}_r\left(p_r-\left(1-p_r\right)\right) - 4\varepsilon (1-p_r) \tag{using \eqref{eq:uc:case3:3}}\\
        &= \hat{a}_r\left(2 p_r-1\right) -4\varepsilon (1-p_r) \nonumber\\
        &\geq 0 -4\varepsilon \geq -10\varepsilon. \tag{$\hat{a}_r < 0$ and $p_r\leq 1/2$}
    \end{align}
    If $r>i^*$. Then $p_r \leq 1 / 4$ by Line 11 in \cref{alg:uc} since $r \neq 1$ and $i^{+} \geq 3$. Hence, by $\hat{a}_r<0$,  the second term of \eqref{eq:uc:case3:2} can then be bounded as
    \begin{align}
        &\sum_{i>i^*} \hat{a}_i p_i+\sum_{i \neq r, i^*} \hat{a}_i p_i+ \hat{a}_r p_r \nonumber\\
        &= \sum_{i>i^*, i \neq r} \hat{a}_i p_i+\sum_{i \neq r, i^*} \hat{a}_i p_i+2 \hat{a}_r p_r \nonumber\\
        &\geq \sum_{i>i^*, i \neq r} (-\hat{a}_r-4\varepsilon) p_i+\sum_{i \neq r, i^*} (-\hat{a}_r-4\varepsilon) p_i+2 \hat{a}_r p_r \tag{using \eqref{eq:uc:pair2:main}}\\
        &= \hat{a}_r\left(2p_r - \sum_{i>i^*,i\neq r} p_i - \sum_{i \neq r, i^*} p_i\right) - 4\varepsilon \left(\sum_{i>i^*, i\neq r} p_i + \sum_{i \neq r, i^*} p_i\right) \nonumber\\
        & = \hat{a}_r(2 p_r-(1-2 p_r)) - 4\varepsilon (1-2p_r) \tag{using \eqref{eq:uc:case3:3}}\\
        &\geq 0 -4\varepsilon \geq -10\varepsilon. \tag{$\hat{a}_r < 0$ and $p_r\leq 1/4$}
    \end{align}
    Thus we conclude that when $\hat{a}_r < 0$, we have
    \begin{align}
        \sum_{i>i^*} \hat{a}_i p_i+\sum_{i \neq r, i^*} \hat{a}_i p_i+ \hat{a}_r p_r \geq -10\varepsilon. \label{eq:uc:case3:4}
    \end{align}
    
    \noindent Combining \eqref{eq:uc:case3:4} and \eqref{eq:uc:case3:5}, we conclude the proof for $i^+ \geq 3$. 
\end{proof}

\begin{remark}
    The assumption of bounded noise plays an essential role in the analysis. Although in the presence of noise, the original property (e.g., pairwise monotonicity) of the function does not hold, with bounded noise bridging the $\hat{f}$ and $f$, we can still have a relaxed version of the desired property, leading to similar approximation ratio with only small error. In many cases, showing the relaxed version require completely new steps.
\end{remark}

\subsection{Regret Bound}

Once we have analysed the robustness  of Algorithm~\ref{alg:uc} and identified the number of function evaluations of Algorithm~\ref{alg:uc} is exactly $nk$, we can employ the framework proposed in \citet{nie23framework} and obtain the following result:

\begin{corollary} 
    For an online non-monotone unconstrained $k$-submodular maximization problem, the expected cumulative $1/2$-regret of C-ETC using \cref{alg:uc} as a sub-routine is at most $\mathcal{O}\left(n k^{\frac{1}{3}} T^\frac{2}{3}\log(T)^\frac{1}{3}\right)$ given $T\geq nk$.
\end{corollary}

\section{Monotone Functions without Constraints} \label{sec:uc-monotone}

In this section, we continue to explore the unconstrained case, but now the objective is monotone. We adopt the offline $\frac{k}{2k-1}$-approximation algorithm proposed in \cite{Iwata2015ImprovedAA}. The pseudo-code is presented in Algorithm~\ref{alg:uc-monotone} in the appendix. Similar to Algorithm~\ref{alg:uc} for the non-monotone case, the algorithm iterates through all $e\in V$ (in an arbitrary but fixed order) and selects the type $i\in [k]$ randomly according to some carefully designed probabilities for each $e\in V$. 

One important note to highlight is the modification made on line 6 of the algorithm. In the original algorithm from \cite{Iwata2015ImprovedAA}, it was stated as $\beta \leftarrow \sum_{i=1}^k y_i^t$. This particular step is not robust to noise since later in line 9, the probability of $y_i^t/\beta$ is assigned to type $i$. The probability becomes meaningless if $\beta < 0$ and this is possible due to noise. We identified that a sufficient modification is to change this step to $\beta \leftarrow \sum_{i=1}^k [y_i]_+^t$. Here, $[x]_+ = x$ if $x\geq 0$, and 0 otherwise. Note that when there is no noise, Algorithm~\ref{alg:uc-monotone} reduces to the original algorithm in \citet{Iwata2015ImprovedAA} since in the monotone case, $y_i\geq 0$ always holds. 

We show the following robustness guarantee of Algorithm~\ref{alg:uc-monotone}:

\begin{proposition} \label{prop:uc:monotone:robust}
    Algorithm~\ref{alg:uc-monotone} for maximizing a monotone $k$-submodular function is a $(\frac{k}{2k-1}, (16-\frac{2}{k})n, nk)$-robust approximation.
\end{proposition}

By Definition~\ref{def:robust}, we aim to show that the output $\bs$ of Algorithm~\ref{alg:uc-monotone} satisfies $\mathbb{E}[f(\bs)] \geq \frac{k}{2k-1}f(\bo)-(16-\frac{2}{k})n\varepsilon$ when the function evaluation is based on surrogate function $\hat{f}$. Note that in this setting, Lemma~\ref{lem:uc:lem} still holds. Thus, it is sufficient to show 
\begin{align}
    \sum_{i=1}^k (\hat{a}_{i^*}-\hat{a}_i) p_i \leq (1-\frac{1}{k}) \sum_{i=1}^k \hat{y}_i p_i + 10\varepsilon, 
\end{align}
which is equivalent to \eqref{eq:uc:condition:main} with $c = 1-\frac{1}{k}$ and $d=10$. The detailed proof is in \cref{sec:prf:uc:monotone}. 

\begin{remark}
    This  modification of $\beta$ becomes significant when dealing with a noisy function oracle $\hat{f}$, where it's possible that $\hat{y}_i < 0$. Our analysis demonstrates that this modification is crucial in handling the case of $\beta=0$, which is a trivial case in the noiseless analysis of \citet{Iwata2015ImprovedAA}. In the presence of noise, the case of $\beta=0$ becomes non trivial. It also plays a pivotal role in another case ($\beta >0$), where we utilized $\beta-[\hat{y}_{i^*}]_+^t \leq \beta$. In the original algorithm, $\beta-\hat{y}_{i^*}^t \leq \beta$ does not necessarily hold due to noise.
\end{remark}

With the robustness guarantee for Algorithm~\ref{alg:uc-monotone}, we can transfer it to the online bandit setting using the framework proposed in \citet{nie23framework}:

\begin{corollary} 
    For an online monotone unconstrained $k$-submodular maximization problem, the expected cumulative $\frac{k}{2k-1}$-regret of C-ETC is at most $\mathcal{O}\left(n k^{\frac{1}{3}}T^\frac{2}{3}\log(T)^\frac{1}{3}\right)$ given $T\geq \max\{nk,\frac{2\sqrt{2}k}{16-\frac{2}{k}}\}$. 
\end{corollary}

\section{Monotone Functions with Individual Size Constraints} \label{sec:is}

In this section, we consider \textit{Individual Size} (IS) constraint. In IS, each type $i$ has a limit $B_i$ on the maximum number of pairs of that type $i$, with $B=\sum_i B_i$ as the total budget. We consider the offline greedy algorithm proposed in \citet{ohsaka2015monotone}.

The pseudo-code is presented in Algorithm~\ref{alg:is} in the appendix. The algorithm builds the solution greedily by iteratively including the element-type pair that yields the largest marginal gain, as long as that pair is still available and the individual budget is not exhausted. We show the following robustness guarantee of Algorithm~\ref{alg:is}:

\begin{proposition} \label{prop:is:robust}
    Algorithm~\ref{alg:is} for maximizing a monotone $k$-submodular function under individual size constraints is a $(\frac{1}{3}, \frac{4}{3}(B+1), nkB)$-robust approximation, where $B$ is the total size limit.
\end{proposition}

The proof is in \cref{sec:prf:is}.
With the robustness guarantee for Algorithm~\ref{alg:is}, we can transfer it to the online bandit setting using the framework proposed in \citet{nie23framework}:

\begin{corollary} 
    For an online monotone $k$-submodular maximization problem under individual size constraints, the expected cumulative $1/3$-regret of C-ETC is at most $\mathcal{O}\left(n^{\frac{1}{3}} k^{\frac{1}{3}} B T^\frac{2}{3}\log(T)^\frac{1}{3}\right)$ given $T\geq nk\max\{1,\frac{3\sqrt{2}B}{2(B+1)}\}$. 
\end{corollary}

\section{Monotone Functions with Matroid Constraints} \label{sec:m:mat}

In this section, we consider matroid constraint with monotone objective functions. We adapt the offline algorithm proposed in \citet{sakaue2017maximizing}, which is depicted in Algorithm~\ref{alg:matroid} in the appendix. The algorithm still builds the solution in a greedy manner. At each iteration, the algorithm constructs available elements $E(\bs)$ given the current solution $\bs$ using an assumed independence oracle, and includes the element-type pair that yields the largest marginal gain. We show the following robustness guarantee of Algorithm~\ref{alg:matroid}:

\begin{proposition} \label{prop:matroid:robust}
    Algorithm~\ref{alg:matroid} for maximizing a monotone $k$-submodular function under a matroid constraint is a $(\frac{1}{2}, M+1, nkM)$-robust approximation, where $M$ is the rank of the matroid.
\end{proposition}

The proof is in \cref{sec:prf:matroid}. 
With the robustness guarantee for Algorithm~\ref{alg:matroid}, we transform it into a CMAB algorithm using the framework proposed in \cite{nie23framework}:

\begin{corollary} \label{thm:monotone:mat}
    For an online monotone $k$-submodular maximization problem under a matroid constraint, the expected cumulative $1/2$-regret of C-ETC is at most $\mathcal{O}\left(n^{\frac{1}{3}} k^{\frac{1}{3}} M T^\frac{2}{3}\log(T)^\frac{1}{3}\right)$ given $T\geq nk\max\{1,\frac{3\sqrt{2}M}{2(M+1)}\}$.
\end{corollary}

One special case of matroid constraint is \textit{Total Size} (TS) constraint \cite{ohsaka2015monotone}. In TS, there is a limited budget on the total number of element-type pairs that can be selected. This is also called uniform matroid in literature. In this case, the rank of the matroid is the total size budget $B$. We can immediately get an online algorithm as a result of Corollary~\ref{thm:monotone:mat}:

\begin{corollary}
    For an online monotone $k$-submodular maximization problem under a TS constraint, the expected cumulative $1/2$-regret of C-ETC is at most $\mathcal{O}\left(n^{\frac{1}{3}} k^{\frac{1}{3}} B T^\frac{2}{3}\log(T)^\frac{1}{3}\right)$ given $T\geq nk\max\{1,\frac{3\sqrt{2}B}{2(B+1)}\}$, using $N_{\mathcal{A}}=nkB$ as an upper bound of the number of function evaluations for Algorithm~\ref{alg:matroid}.  
\end{corollary}

\begin{remark}
    In \cref{sec:is,sec:m:mat}, although marginal gains are not guaranteed to be non-negative, we prove that there is no need to modify the algorithms as in \cref{sec:uc-monotone}. This is due to the inherent design of the algorithms. In Algorithm~\ref{alg:uc-monotone}, probabilities are assigned to each item in proportion to their marginal gains for each type (Line 9), so it is necessary to ensure that the sum of the marginal gains is non-negative. However, \cref{alg:is,alg:matroid} select item-type pairs by directly comparing marginal gains, regardless of whether they are positive or negative. In cases where monotonicity does not hold due to noise, we use bounded error analysis to derive analogous properties. 
\end{remark}

\section{Non-Monotone Functions with  Matroid Constraints}\label{sec:nm:mat}

In \citet{sun2022maximize}, it is shown that Algorithm~\ref{alg:matroid} in appendix can achieve $1/3$ approximation ratio even the objective function is non-monotone. We show the following robustness result when Algorithm~\ref{alg:matroid} is fed with surrogate function $\hat{f}$.

\begin{proposition} \label{prop:matroid:robust2}
    Algorithm~\ref{alg:matroid} for maximizing a non-monotone $k$-submodular function under a matroid constraint is a $(\frac{1}{3}, \frac{4}{3}(M+1),nkM)$-robust approximation, where $M$ is the rank of the matroid.
\end{proposition}
The proof is in \cref{sec:prf:matroid2}.
We use the offline-to-online transformation framework proposed in \citet{nie23framework} to adapt Algorithm~\ref{alg:matroid}: 

\begin{corollary} 
    For an online non-monotone $k$-submodular maximization problem under a matroid constraint, the expected cumulative $1/3$-regret of C-ETC is at most $\mathcal{O}\left(n^{\frac{1}{3}} k^{\frac{1}{3}} M T^\frac{2}{3}\log(T)^\frac{1}{3}\right)$ given $T\geq nk\max\{1,\frac{3\sqrt{2}M}{2(M+1)}\}$.
\end{corollary}

\section{Evaluations}
In this section, we empirically evaluate the performance of our proposed methods in the context of online influence maximization with $k$ topics. This problem is formulated as a monotone $k$-submodular bandit problem (detailed below). As there are no directly comparable baselines in the literature, we select naive UCB (NaiveUCB) and random selection (Random) as benchmarks, with UCB implemented using the standard UCB1 algorithm.

\textbf{Online Influence Maximization with $k$ Topics:}
The problem involves a social network represented as a directed graph $G=(V, E)$, where $V$ is the set of nodes and $E$ is the set of edges. Each edge $(u, v) \in E$ has associated weights $p_{u,v}^i$, $i \in [k]$, where $p_{u,v}^i$ denotes the influence strength from user $u$ to $v$ on topic $i$. The goal is to maximize the number of users influenced by one or more topics. We adopt the \textit{$k$-topic independent cascade} ($k$-IC) model from \citet{ohsaka2015monotone}, which generalizes the independent cascade model \cite{kempe2003maximizing}. 

Specifically, the influence spread $\sigma:(k+1)^V \rightarrow \mathbb{R}_{+}$ in the $k$-IC model is defined as $\sigma(S)=\mathbb{E}\left[\left|\bigcup_{i \in[k]} A_i\left(U_i(S)\right)\right|\right]$, where $A_i\left(U_i(S)\right)$ is a random variable representing the set of influenced nodes in the diffusion process of topic $i$. As shown in \citet{ohsaka2015monotone}, $\sigma$ is monotone $k$-submodular. Given a directed graph $G=(V, E)$, edge probabilities $\{p_{u,v}^i \mid (u,v) \in E, i \in [k]\}$, the task is to select a seed set $S \in (k+1)^V$ that maximizes $\sigma(S)$, subject to some constraints (e.g., total size constraints).

\textbf{Experiment settings:} We use the ego-facebook network \cite{McAuley2012LearningTD}. After applying a community detection algorithm \cite{Blondel2008FastUO} to extract a subgraph, we convert it into a directed graph consisting of 350 users and 2,845 edges. Three sets of edge weights, corresponding to three topics ($k=3$), are generated: Uniform(0, 0.2), Normal(0.1, 0.05), and Exponential(10), with the Normal and Exponential distributions truncated to [0, 1]. We evaluate our algorithms under both TS and IS constraints. The TS budget is set to $B=6$, while for IS constraints, the budgets for the three topics are all 2. As greedy algorithms offer $1/2$ and $1/3$ approximations for TS and IS respectively \cite{ohsaka2015monotone}, we use an offline greedy algorithm to compute the $\alpha$-optimal value, with 100 diffusion simulations approximating the true function value. Since all cases are monotone, NaiveUCB and Random restrict their search space to actions that exhaust the budget.  Means and standard deviations are calculated over 10 independent runs.

\textbf{Results:} The instantaneous rewards over $10^{4}$ time steps is shown in \cref{fig:im}. Under both TS and IS constraints, the results can be summarized as follows. In the early stages (before 4,000), ETCG explores suboptimal actions that do not use the full budget, leading to worse initial instantaneous rewards. However, ETCG catches up in later stages and achieves lower cumulative regret over the entire time horizon. We note that NaiveUCB does UCB on all actions that exhaust budgets, which takes approximately ${60 \choose 6} \sim 5\times 10^{8}$ for TS and ${20 \choose 2}^3 \sim 7\times10^{7}$ for IS, respectively, to even explore each action once, resulting a bad performance.

\begin{figure}[t]
    \centering
    \begin{subfigure}[b]{0.48\linewidth}
        \centering
        \includegraphics[width=1.8in]{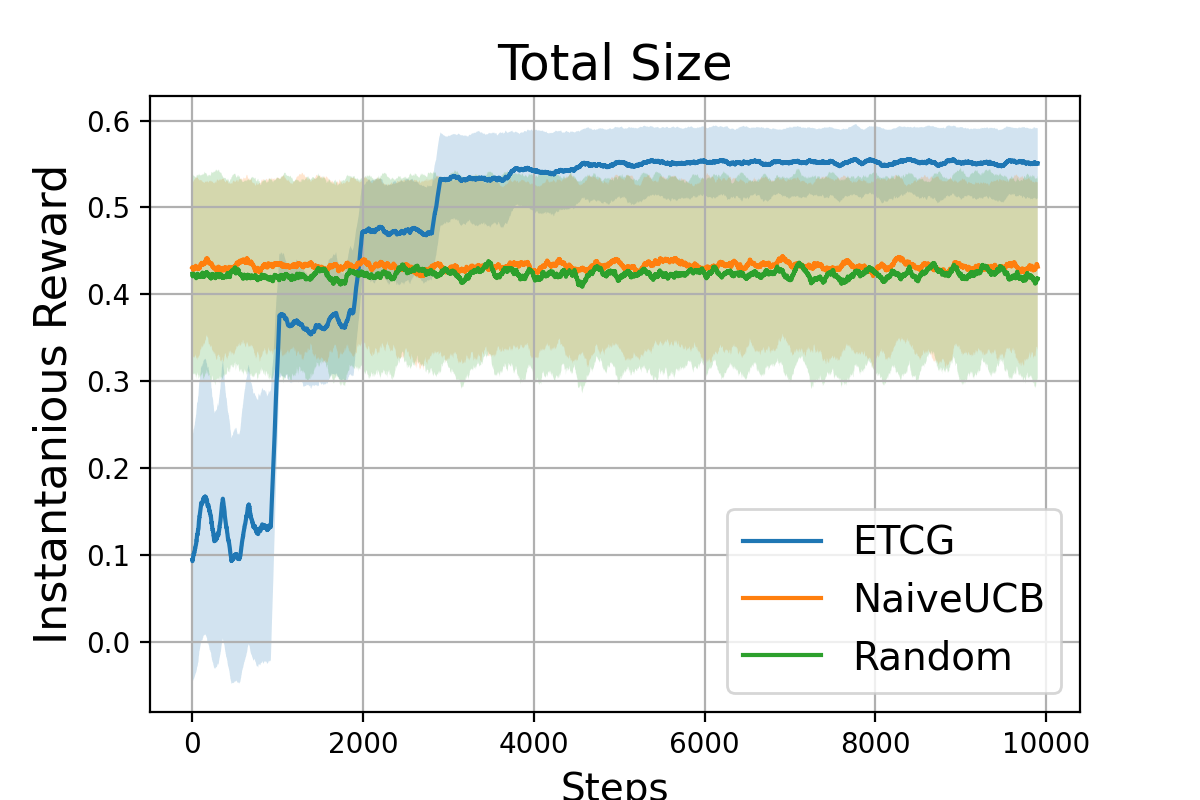}
    \end{subfigure}
    \begin{subfigure}[b]{0.48\linewidth}
        \centering
        \includegraphics[width=1.8in]{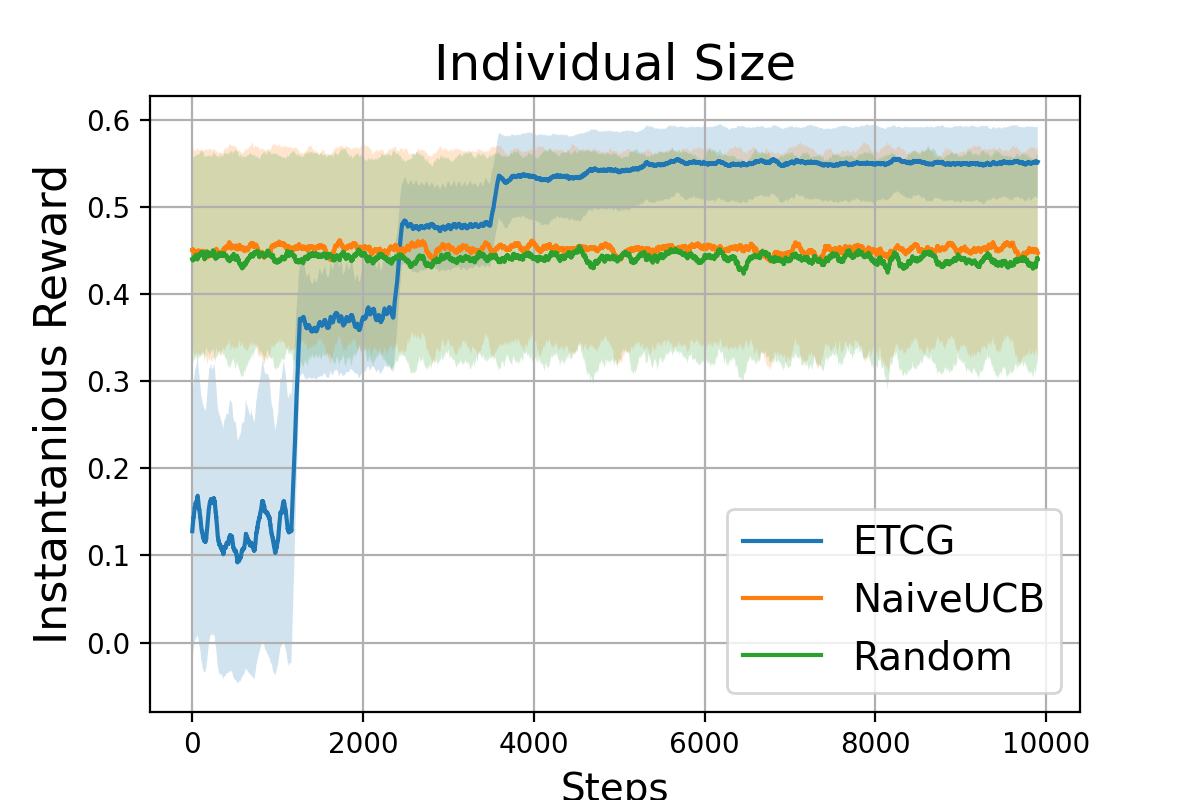}
    \end{subfigure}
    \caption{\normalsize Instantaneous Rewards on Influence Maximization experiments.} 
    \label{fig:im}
\end{figure}

\section{Conclusion}
In this paper, we investigate the problem of online $k$-submodular maximization problem under bandit feedback. Utilizing the framework proposed by \citet{nie23framework}, we propose CMAB algorithms through adapting various offline algorithms and analyzing their robustness. We obtain sublinear regret for online $k$-submodular maximization under bandit feedback under various settings, including monotone and non-monotone objectives without constraint, monotone objectives with individual size constraint, monotone and non-monotone objectives with matroid constraint. Numerical experiments on online influence maximization with $k$ topics were conducted to evaluate the effectiveness of our methods. 

\bibliography{refs}
\bibliographystyle{icml2024}

\newpage
\appendix
\onecolumn
\section{Missing Proofs}
\subsection{Proof of Lemma~\ref{lem:uc:lem}} \label{sec:prf:uc:lem}
We restate the lemma here:
\begin{lemma}
    Let $c, d \in \mathbb{R}_{+}$. Conditioning on $\bs^{(j-1)}$, suppose that
    \begin{align}
        \sum_{i=1}^k\left(\hat{a}_{i^*}^{(j)}-\hat{a}_i^{(j)}\right) p_i^{(j)} \leq c\left(\sum_{i=1}^k \hat{y}_i^{(j)} p_i^{(j)}\right) + d\varepsilon \label{eq:uc:condition}
    \end{align}
    holds for each $j$ with $1 \leq j \leq n$, where $i^*=\bo(e^{(j)})$. Then $\mathbb{E}[f(\bs)] \geq \frac{1}{1+c} f(\bo) -(2c+d+4)n\varepsilon$.
\end{lemma}

\begin{proof}
    Fix $j \in [n]$.  
    Conditioning on $\bs^{(j-1)}$, element $e^{(j)}$ will be randomly assigned a type $i\in[k]$ following probabilities described in Lines 9-11 in the pseudocode.  We first calculate the conditional expected difference between $f(\bo^{(j-1)})$ and $f(\bo^{(j)})$.  
    Let $i^*$ denote  the type of $e^{(j)}$ in $\bo$ (and thus in $\bo^{(j-1)}$). By construction, since $\bt^{(j-1)}$ is obtained from $\bo^{(j)}$ by changing $\bo^{(j)}(e^{(j)})$ (which is the same as $\bs(e^{(j)})$ and thus is $i$ with probability $p_i^{(j)}$) with 0, we have 
    \begin{align}
        \mathbb{E}[f(\bo^{(j)})-f(\bt^{(j-1)})| \bs^{(j-1)}] = \sum_{i=1}^k \Delta_{e^{(j)}, i} f(\bt^{(j-1)}) p_i^{(j)} = \sum_{i=1}^k a_i^{(j)} p_i^{(j)}. \label{eq:uc:lem:construct1}
    \end{align}
    Also, $\bo^{(j-1)}$ can be obtained from $\bt^{(j-1)}$ by replacing $\bt^{(j-1)}(e^{(j)})$ (which is 0) with $\bo^{(j)}(e^{(j)})$ (which is $i^*$), we have
    \begin{align}
        f(\bo^{(j-1)})-f(\bt^{(j-1)}) = a_{i^*}^{(j)}.\label{eq:uc:lem:construct2}
    \end{align} Thus
    \begin{align}
        &\mathbb{E}[f(\bo^{(j-1)})-f(\bo^{(j)}) | \bs^{(j-1)}] \nonumber\\ 
        =&\mathbb{E}[f(\bo^{(j-1)})-f(\bt^{(j-1)})+f(\bt^{(j-1)})-f(\bo^{(j)}) | \bs^{(j-1)}] \nonumber\\
        =&\sum_{i=1}^k\left(a_{i^*}^{(j)}-a_i^{(j)}\right) p_i^{(j)} \tag{using \eqref{eq:uc:lem:construct1} and \eqref{eq:uc:lem:construct2}}\\
        \leq& \sum_{i=1}^k\left(\hat{a}_{i^*}^{(j)}-\hat{a}_i^{(j)} + 4\varepsilon\right)p_i^{(j)} \tag{definition of $\hat{f}$}\\
        =&  \sum_{i=1}^k\left(\hat{a}_{i^*}^{(j)}-\hat{a}_i^{(j)}\right)p_i^{(j)} + 4\varepsilon \label{eq:uc:lem:1}
    \end{align}
    where the last equality is due to $\sum_i p_i^{(j)}=1$. $\bs^{(j)}$ can be obtained from $\bs^{(j-1)}$ by replacing $\bs^{(j-1)}(e^{(j)})$ (which is 0) with $\bs^{(j)}(e^{(j)})$ (which is $i$ with probability $p_i^{(j)}$), we have
    \begin{align}
        &\mathbb{E}[f(\bs^{(j)})-f(\bs^{(j-1)}) | \bs^{(j-1)}] \nonumber\\
        =&\sum_{i=1}^k y_i^{(j)} p_i^{(j)} \tag{by construction}\\
        \geq& \sum_{i=1}^k (\hat{y}_i^{(j)} - 2\varepsilon)p_i^{(j)} \tag{definition of $\hat{f}$}\\
        =& \sum_{i=1}^k \hat{y}_i^{(j)} p_i^{(j)}-2\varepsilon, \label{eq:uc:lem:2}
    \end{align}
    where both expectations are taken over the randomness of the algorithm. Combining \eqref{eq:uc:lem:1}, \eqref{eq:uc:lem:2} and \eqref{eq:uc:condition}, we have 
    \begin{align}
        \mathbb{E}[f(\bo^{(j-1)})-f(\bo^{(j)}) | \bs^{(j-1)}] \leq c \mathbb{E}[f(\bs^{(j)})-f(\bs^{(j-1)}) | \bs^{(j-1)}] + (2c+d+4) \varepsilon. \label{eq:uc:lem:3}
    \end{align}
    Note that $\bo^{(0)}=\bo$ and $\bo^{(n)}=\bs$ by construction. Hence
    \begin{align}
        f(\bo)-\mathbb{E}[f(\bs)] &=\sum_{j=1}^n \mathbb{E}[f(\bo^{(j-1)})-f(\bo^{(j)})] \tag{by construction}\\
        &=\sum_{j=1}^n \mathbb{E}[\mathbb{E}[f(\bo^{(j-1)})-f(\bo^{(j)})| \bs^{(j-1)}] ] \nonumber\\
        &=\sum_{j=1}^n \mathbb{E}\left[ c \mathbb{E}[f(\bs^{(j)})-f(\bs^{(j-1)})| \bs^{(j-1)}] +  (2c+d+4)n \varepsilon] \right] \tag{using \eqref{eq:uc:lem:3}}\\
        & \leq c\left(\sum_{j=1}^n \mathbb{E}[f(\bs^{(j)})-f(\bs^{(j-1)})]\right) +  (2c+d+4)n \varepsilon \nonumber\\
        & =c(\mathbb{E}[f(\bs)]-f(\mathbf{0})) + (2c+d+4)n \varepsilon \nonumber\\
        &= c \mathbb{E}[f(\bs)] + (2c+d+4)n \varepsilon, \tag{$f(\boldsymbol{0})=0$}
    \end{align}
    and we get the statement.
\end{proof}

\setcounter{algorithm}{1}

\begin{algorithm}[t]
\caption{A randomized algorithm for $k$-submodular maximization without constraints \citep{Iwata2015ImprovedAA}}
\label{alg:uc}
\begin{algorithmic}[1]
    \STATE {\bfseries Input:} ground set $V$, value oracle access for $f$  
    \STATE {\bfseries Output:} a vector $\bs \in (k + 1)^V$.
    \STATE Initialize $\bs \leftarrow \mathbf{0}_n$.
    \FOR{$e \in V$} 
        \STATE $y_i\leftarrow \Delta_{e,i}f(\bs)$ for $i\in [k]$.
        \STATE w.l.o.g. assume $y_1\geq y_2\geq \ldots\geq y_k$.
        \STATE $i^+\leftarrow 
        \begin{cases}
            \text{max. } i \text{ such that } y_i>0 & \text{if } y_1>0, \\
            0 & \text{otherwise.}
        \end{cases}$
        \STATE Initialize $p \gets \mathbf{0}_k$.
        \STATE \textbf{if } $i^+\leq 1$,\ 
        $p_1 \gets 1$.
        \STATE \textbf{else if } $i^+= 2$, 
        for $i\in\{1,2\}$\ \ $p_i \gets \frac{y_i}{(y_1+y_2)}$.
        \STATE \textbf{else} $p_{i^+} \gets \left(\frac{1}{2}\right)^{i^+ -1}$  
        and for $i<i^+$,  $p_{i} \gets \left(\frac{1}{2}\right)^{i}$ 
        \STATE Choose $\bs(e) \sim \mathbb{P}(\bs(e)=i)=p_i$ for all $i\in[k]$.
    \ENDFOR
    \STATE {\bfseries Return} $\bs$.
\end{algorithmic}
\end{algorithm}

\subsection{Proof of Proposition~\ref{thm:nmuc:robust}} \label{sec:prf:uc}
We repeat some notations that we will use in our analysis that have already been introduced in the main text.

Let $\bo$ be an optimal solution with $\operatorname{supp}(\bo)=V$ (by \cref{prop:Iwata2015ImprovedAA} such a solution exists). 
Let $\bs$ be the output of the algorithm with $|\operatorname{supp}(\bs)|=n$ using surrogate function $\hat{f}$. 
We consider the $j$-th iteration of the algorithm, and let $e^{(j)}$ be the element of $V$ considered in the $j$-th iteration, $p_i^{(j)}$ be the probability that $i$-th type is chosen in the $j$-th iteration, and $\bs^{(j)}$ be the solution after the $i$-th iteration, where $\bs^{(0)}=\mathbf{0}$. Also for $0 \leq j \leq n$, let $\bo^{(j)}$ be the element in $\{0,1, \ldots, k\}^V$ obtained from $\bo$ by replacing the coordinates on $\operatorname{supp}(\bs^{(j)})$ with those of $\bs^{(j)}$, and for $1 \leq j \leq n$ let $\bt^{(j-1)}$ be the element in $\{0,1, \ldots, k\}^V$ obtained from $\bo^{(j)}$ by changing $\bo^{(j)}(e^{(j)})$ with 0.

The subsequent analyses will involve comparing different marginal gains, which we next introduce. 
For $i \in[k]$, let $y_i^{(j)}=\Delta_{e^{(j)}, i} f(\bs^{(j-1)})$, $\hat{y}_i^{(j)}=\Delta_{e^{(j)}, i} \hat{f}(\bs^{(j-1)})$ and let $a_i^{(j)}=\Delta_{e^{(j)}, i} f(\bt^{(j-1)})$, $\hat{a}_i^{(j)}=\Delta_{e^{(j)}, i} \hat{f}(\bt^{(j-1)})$. Due to pairwise monotonicity, we have $y_i^{(j)}+y_{i'}^{(j)} \geq 0$ and  $a_i^{(j)}+a_{i'}^{(j)} \geq 0$ for all $i, i' \in[k]$ with $i \neq i'$, and thus, while the surrogate function $\hat{f}$ does not necessarily have pairwise monotonicity, we have 

\begin{align}
    \hat{y}_i^{(j)}+\hat{y}_{i'}^{(j)} \geq -4 \varepsilon &, \label{eq:uc:pair1}\\
    \hat{a}_i^{(j)}+\hat{a}_{i'}^{(j)} \geq -4\varepsilon &. \label{eq:uc:pair2}
\end{align}
for all $i, i' \in[k]$ with $i \neq i'$ since both inequalities involve a sum of four function values of $\hat{f}$ which can be off by at most $\epsilon$ with respect to $f$. 
Also from $\boldsymbol{s}^{(j)} \preceq \boldsymbol{t}^{(j)}$ which holds by construction,  orthant submodularity implies $y_i^{(j)} \geq a_i^{(j)}$ for all $i \in[k]$ and thus,
\begin{align}
    \hat{y}_i^{(j)} \geq \hat{a}_i^{(j)} -4\varepsilon. \label{eq:uc:orthant}
\end{align}
for all $i \in[k]$. Without loss of generality, we assume $f(\boldsymbol{0})=0$.

Now we are ready to prove Proposition~\ref{thm:nmuc:robust}. 
\begin{proof}[Proof of Proposition~\ref{thm:nmuc:robust}] 
It is trivial to see the number of value oracle queries required is $nk$, as it iterates through all item-type pairs. The following proof will be dedicated to show the $\alpha$ and $\delta$ terms in Definition~\ref{def:robust}.

By Definition~\ref{def:robust}, we aim to show that the output $\bs$ of Algorithm~\ref{alg:uc} satisfies $\mathbb{E}[f(\bs)] \geq \frac{1}{2}f(\bo)-20n\varepsilon$ when the function evaluation is based on surrogate function $\hat{f}$. By Lemma~\ref{lem:uc:lem}, our goal is to show
    \begin{align}
        \sum_{i=1}^k (\hat{y}_i+\hat{a}_i) p_i \geq \hat{a}_{i^*} - 14\varepsilon, \label{eq:uc:non-mono}
    \end{align}
    which is equivalent to \eqref{eq:uc:condition:main} in main text with $c=1$ and $d=14$. For simplicity of the description we shall omit the superscript $(j)$ if it is clear from the context. Recall that $\hat{y}_i+\hat{y}_{i'} \geq -4\varepsilon$ and $\hat{a}_i+\hat{a}_{i'} \geq -4\varepsilon$ for $i, i' \in[k]$ with $i \neq i'$, and $\hat{y}_i \geq \hat{a}_i - 4\varepsilon$ for $i \in[k]$ (c.f. \eqref{eq:uc:pair1},\eqref{eq:uc:pair2} and \eqref{eq:uc:orthant}). We break up the analysis to three cases: $i^{+} \leq 1$, $i^{+} = 2$ and $i^{+} \geq 3$. 
    
    \paragraph{Case 1:} If $i^{+} \leq 1$, then by Line 9 in \cref{alg:uc} $p_1=1$ and $p_i=0$ for $i>1$.  So \eqref{eq:uc:non-mono} specializes to $\hat{a}_1+\hat{y}_1 \geq \hat{a}_{i^*}-14\varepsilon$.
    Since $\hat{y}_i+\hat{y}_{i'} \geq -4\varepsilon$ for $i, i' \in[k]$ and $\hat{y}'s$ are ranked in decreasing order, 
    we have $\hat{y}_1+\hat{y}_{i'} \geq -4\varepsilon$ and $\hat{y}_1 \geq \hat{y}_{i'}$ for $i' \in[k]$. 
    Summing these two inequalities we have $\hat{y}_1 \geq -2\varepsilon$. 
    \paragraph{Sub-case 1.a:} If $i^*=1$, $$\hat{a}_1+\hat{y}_1 = \hat{a}_{i^*}+\hat{y}_1 \geq \hat{a}_{i^*}-2\varepsilon\geq \hat{a}_{i^*}-14\varepsilon,$$
    showing the statement. 
    \paragraph{Sub-case 1.b:} If $i^* \neq 1$, from $i^+ \leq 1$ and \eqref{eq:uc:orthant}, $0 \geq \hat{y}_{i^*} \geq \hat{a}_{i^*}-4\varepsilon$. 
    We also have $\hat{a}_1+\hat{a}_{i^*} \geq -4\varepsilon$ by \eqref{eq:uc:pair2}. 
    Hence by summing these two inequalities we have $\hat{a}_1 \geq -8\varepsilon$. 
    Combining with $\hat{y}_1 \geq -2\varepsilon$ we have $$\hat{a}_1+\hat{y}_1 \geq -10\varepsilon \geq \hat{a}_{i^*}-14\varepsilon,$$
    showing the statement.
    \paragraph{Case 2:} If $i^{+}=2$, , then by Line 10 in \cref{alg:uc}, \eqref{eq:uc:non-mono} specializes to 
    \[(\hat{a}_1+\hat{y}_1) \hat{y}_1+(\hat{a}_2+\hat{y}_2) \hat{y}_2 \geq \hat{a}_{i^*}(\hat{y}_1+\hat{y}_2)-14\varepsilon (\hat{y}_1+\hat{y}_2).\]
    In this case, we have $\hat{y}_1, \hat{y}_2 > 0$ by definition of $i^+$. 
    \paragraph{Sub-case 2.a:} $i^*\leq 2$. Now $(\hat{a}_1+\hat{y}_1) \hat{y}_1+(\hat{a}_2+\hat{y}_2) \hat{y}_2=\hat{a}_1 \hat{y}_1+\hat{a}_2 \hat{y}_2+(\hat{y}_1-\hat{y}_2)^2+2 \hat{y}_1 \hat{y}_2 \geq \hat{a}_1 \hat{y}_1+\hat{a}_2 \hat{y}_2+2 \hat{y}_1 \hat{y}_2$. If $i^*=1$, then 
    \begin{align}
        &\hat{a}_1 \hat{y}_1+\hat{a}_2 \hat{y}_2+2 \hat{y}_1 \hat{y}_2 \nonumber\\
        &\geq \hat{a}_1 \hat{y}_1+\hat{a}_2 \hat{y}_2+2 (\hat{a}_1 - 4\varepsilon) \hat{y}_2 \tag{using \eqref{eq:uc:orthant}}\\
        &= \hat{a}_1(\hat{y}_1+\hat{y}_2)+(\hat{a}_2+\hat{a}_1) \hat{y}_2 - 8\varepsilon \hat{y}_2 \nonumber\\
        &\geq \hat{a}_1(\hat{y}_1+\hat{y}_2)- 12\varepsilon \hat{y}_2 \tag{using \eqref{eq:uc:pair2}}\\
        &\geq \hat{a}_1(\hat{y}_1+\hat{y}_2)- 12\varepsilon (\hat{y}_1+\hat{y}_2)\tag{$\hat{y}_1 \geq 0$}\\
        &= \hat{a}_{i^*}(\hat{y}_1+\hat{y}_2)-  12\varepsilon (\hat{y}_1+\hat{y}_2) \tag{$i^* = 1$}\\
        &\geq \hat{a}_{i^*}(\hat{y}_1+\hat{y}_2)-  14\varepsilon (\hat{y}_1+\hat{y}_2)
    \end{align} 
    as required. By a similar calculation the claim follows if $i^*=2$.
    \paragraph{Sub-case 2.b:} If $i^* \geq 3$, then $0 \geq \hat{y}_{i^*} \geq \hat{a}_{i^*}-4\varepsilon$ by definition of $i^+$ and \eqref{eq:uc:orthant}. 
    By \eqref{eq:uc:pair2}, we also have that $\hat{a}_1+\hat{a}_{i^*} \geq -4\varepsilon$ and $\hat{a}_2+\hat{a}_{i^*} \geq -4\varepsilon$.
    Thus,
    \begin{align}
        &(\hat{a}_1+\hat{y}_1) \hat{y}_1+(\hat{a}_2+\hat{y}_2) \hat{y}_2 \nonumber\\
        &= \hat{a}_1 \hat{y}_1 + \hat{y}_1^2 + \hat{a}_2 \hat{y}_2 + \hat{y}_2^2 \nonumber\\
        & \geq \hat{a}_1 \hat{y}_1 + \hat{a}_2 \hat{y}_2  \tag{$\hat{y}_1,\hat{y}_2 > 0$ in Case 2}\\
        &\geq (-\hat{a}_{i^*}-4\varepsilon)(\hat{y}_1 + \hat{y}_2) \nonumber\\
        &\geq (-\hat{a}_{i^*}-4\varepsilon)(\hat{y}_1 + \hat{y}_2) + 2(\hat{a}_{i^*}-4\varepsilon)(\hat{y}_1+\hat{y}_2) \tag{$\hat{y}_1,\hat{y}_2 \geq 0$ and $\hat{a}_{i^*}-4\varepsilon \leq 0$}\\
        &= \hat{a}_{i^*}(\hat{y}_1 + \hat{y}_2) - 12 \varepsilon(\hat{y}_1+\hat{y}_2) \nonumber\\
        &\geq \hat{a}_{i^*}(\hat{y}_1 + \hat{y}_2) - 14 \varepsilon(\hat{y}_1+\hat{y}_2)
    \end{align}
    as required.
    \paragraph{Case 3:} When $i^{+} \geq 3$, our goal is to show
    \begin{align}
        \sum_{i=1}^k (\hat{y}_i+\hat{a}_i) p_i \geq \hat{a}_{i^*} - 14\varepsilon, \label{eq:uc:non-mono:supp}
    \end{align} 
    which is equivalent to \eqref{eq:uc:condition:main} with $c=1$ and $d=14$. By the ordering of $\{\hat{y}_i\}_{i\in[k]}$ and \eqref{eq:uc:orthant:main}, for $i\leq i^*$ we have
    \begin{align}
        \hat{y}_i \geq \hat{y}_{i^*} \geq \hat{a}_{i^*}-4\varepsilon. \label{eq:uc:case3:1:supp}
    \end{align}
    Denote $r \in \argmin_{i \in[k]} \hat{a}_i $. 
    
    \textbf{Sub-case 3.a:} If $r=i^*$, we have $\sum_i \hat{a}_i p_i \geq \hat{a}_{i^*}(\sum_i p_i)=\hat{a}_{i^*}$. Since $\sum_i \hat{y}_i p_i \geq 0$ and since from Line 11 in \cref{alg:uc},   a positive  probability $p_i>0$ is only assigned to positive types $i$ with positive $\hat{y}_i>0$ values,
    \begin{align*}
        \sum_{i=1}^k (\hat{y}_i+\hat{a}_i) p_i 
        \geq 0+ \hat{a}_{i^*} 
        \geq  \hat{a}_{i^*} - 14\varepsilon.
    \end{align*}    
    Thus \eqref{eq:uc:non-mono:supp} follows. 

    \textbf{Sub-case 3.b:} If $r \neq i^*$ and $i^* \geq i^{+}$, since from Line 11 in \cref{alg:uc},   a positive  probability $p_i>0$ is only assigned to positive $\hat{y}$ values, we have that $\sum_i \hat{y}_i p_i=\sum_{i \leq i^{+}} \hat{y}_i p_i \geq \sum_{i \leq i^{+}} (\hat{a}_{i^*} -4\varepsilon)p_i=\hat{a}_{i^*}-4\varepsilon$ by \eqref{eq:uc:case3:1:supp}. 
    When $\hat{a}_r\geq 0$, since $\hat{a}_i\geq \hat{a}_r$ for any $i\neq r$ by definition of $r$, we have $\sum_i \hat{a}_i p_i\geq 0$. 
    When $\hat{a}_r < 0$,
    \begin{align}
        \sum_i \hat{a}_i p_i 
        &\geq \sum_{i \neq r}(-\hat{a}_r -4\varepsilon)  p_i+\hat{a}_r p_r\tag{by \eqref{eq:uc:pair2:main}}\\
        &= \hat{a}_r \left(p_r-\sum_{i \neq r}p_i \right) - 4\varepsilon \sum_{i \neq r}p_i \nonumber\\
        &\geq 0 - 4\varepsilon, \nonumber
    \end{align}
    where the last inequality follows from $\hat{a}_r<0$, $\sum_{i \neq r} p_i \geq p_r$ and $\sum_{i \neq r}p_i\leq 1$, and $\sum_{i \neq r} p_i \geq p_r$ follows from Line 11 in \cref{alg:uc} that the largest single probability assigned is $\frac{1}{2}$.  
    Thus,
        \begin{align*} 
        \sum_{i=1}^k (\hat{y}_i+\hat{a}_i) p_i
        \geq (\hat{a}_{i^*} - 4\varepsilon) + (- 4\varepsilon) \geq  \hat{a}_{i^*} - 14\varepsilon.
    \end{align*} 
    Therefore \eqref{eq:uc:non-mono:supp} holds.
    
    \textbf{Sub-case 3.c:} If $r \neq i^*$ and $i^*<i^{+}$,  we have
    \begin{align}
        \sum_{i=1}^k (\hat{y}_i+\hat{a}_i) p_i &= \sum_{i=1}^k \hat{y}_ip_i + \sum_{i=1}^k \hat{a}_ip_i \nonumber\\
        & \geq \sum_{i \leq i^*} (\hat{a}_{i^*} -4\varepsilon) p_i+\sum_{i>i^*} (\hat{a}_i-4\varepsilon) p_i+\sum_{i=1}^k \hat{a}_ip_i  \tag{using \eqref{eq:uc:case3:1:supp} and \eqref{eq:uc:orthant:main}}\\
        &=\sum_{i \leq i^*} \hat{a}_{i^*} p_i+\sum_{i>i^*} \hat{a}_i p_i+\sum_{i=1}^k \hat{a}_i p_i - 4\varepsilon \nonumber\\
        & =\left(\sum_{i<i^*} \hat{a}_{i^*} p_i+2 \hat{a}_{i^*} p_{i^*}\right) +\left(\sum_{i>i^*} \hat{a}_i p_i+\sum_{i \neq r, i^*} \hat{a}_i p_i+ \hat{a}_r p_r\right) - 4 \varepsilon. \label{eq:uc:case3:2:supp}
    \end{align}
    The first term equals $\hat{a}_{i^*}$ by construction of $p_i$'s and since $i^*<i^{+}$.
    Hence it suffices to show that the second term of \eqref{eq:uc:case3:2:supp} is greater than or equal to $-10\varepsilon$. 

    \textbf{Subsub-case 3.c.I:}   
    First we consider $\hat{a}_r \geq 0$. 
    By definition of $r$, for all $i\in[k]$, $\hat{a}_i \geq \hat{a}_r \geq 0$. The second term of \eqref{eq:uc:case3:2:supp} can then be bounded as 
    \begin{align}
        \sum_{i>i^*} \hat{a}_i p_i+\sum_{i \neq r, i^*} \hat{a}_i p_i+ \hat{a}_r p_r \geq \hat{a}_r \left(p_r + \sum_{i>i^*} p_i + \sum_{i \neq r, i^*} p_i\right) \geq 0 \geq -10\varepsilon. \label{eq:uc:case3:5:supp}
    \end{align}

    \textbf{Subsub-case 3.c.II:} 
    Second we consider $\hat{a}_r<0$. Since $i^*<i^{+}$, we have
    \begin{align}
        \sum_{i>i^*} p_i+\sum_{i \neq r, i^*} p_i &= \sum_{i^*+1\leq i\leq i^+ -1} p_i + p_{i^+}+\sum_{i \neq r, i^*} p_i \nonumber\\
        &=\sum_{i=i^*+1}^{i^+-1} \left(\frac{1}{2}\right)^{i} + \left(\frac{1}{2}\right)^{i^+-1} +\sum_{i \neq r, i^*} p_i \tag{by construction of $p_i$'s}\\
        &=\left(\frac{1}{2}\right)^{i^*}+\sum_{i \neq r, i^*} p_i \tag{geometric sum}\\
        &=p_{i^*}+\sum_{i \neq r, i^*} p_i =1-p_r. \label{eq:uc:case3:3:supp}
    \end{align}

    Therefore, if $r<i^*$, we get
    \begin{align}
        &\sum_{i>i^*} \hat{a}_i p_i+\sum_{i \neq r, i^*} \hat{a}_i p_i+ \hat{a}_r p_r \nonumber\\
        &\geq \sum_{i>i^*} (-\hat{a}_r -4\varepsilon) p_i+\sum_{i \neq r, i^*} (-\hat{a}_r -4\varepsilon) p_i+ \hat{a}_r p_r \tag{using \eqref{eq:uc:pair2:main} }\\
        &= \hat{a}_r\left(p_r - \sum_{i>i^*} p_i - \sum_{i \neq r, i^*} p_i\right) - 4\varepsilon \left(\sum_{i>i^*} p_i + \sum_{i \neq r, i^*} p_i\right) \nonumber\\
        &=\hat{a}_r\left(p_r-\left(1-p_r\right)\right) - 4\varepsilon (1-p_r) \tag{using \eqref{eq:uc:case3:3:supp}}\\
        &= \hat{a}_r\left(2 p_r-1\right) -4\varepsilon (1-p_r) \nonumber\\
        &\geq 0 -4\varepsilon \geq -10\varepsilon. \tag{$\hat{a}_r < 0$ and $p_r\leq 1/2$}
    \end{align}
    If $r>i^*$. Then $p_r \leq 1 / 4$ by Line 11 in \cref{alg:uc} since $r \neq 1$ and $i^{+} \geq 3$. Hence, by $\hat{a}_r<0$,  the second term of \eqref{eq:uc:case3:2:supp} can then be bounded as
    \begin{align}
        &\sum_{i>i^*} \hat{a}_i p_i+\sum_{i \neq r, i^*} \hat{a}_i p_i+ \hat{a}_r p_r \nonumber\\
        &= \sum_{i>i^*, i \neq r} \hat{a}_i p_i+\sum_{i \neq r, i^*} \hat{a}_i p_i+2 \hat{a}_r p_r \nonumber\\
        &\geq \sum_{i>i^*, i \neq r} (-\hat{a}_r-4\varepsilon) p_i+\sum_{i \neq r, i^*} (-\hat{a}_r-4\varepsilon) p_i+2 \hat{a}_r p_r \tag{using \eqref{eq:uc:pair2:main}}\\
        &= \hat{a}_r\left(2p_r - \sum_{i>i^*,i\neq r} p_i - \sum_{i \neq r, i^*} p_i\right) - 4\varepsilon \left(\sum_{i>i^*, i\neq r} p_i + \sum_{i \neq r, i^*} p_i\right) \nonumber\\
        & = \hat{a}_r(2 p_r-(1-2 p_r)) - 4\varepsilon (1-2p_r) \tag{using \eqref{eq:uc:case3:3:supp}}\\
        &\geq 0 -4\varepsilon \geq -10\varepsilon. \tag{$\hat{a}_r < 0$ and $p_r\leq 1/4$}
    \end{align}
    Thus we conclude that when $\hat{a}_r < 0$, we have
    \begin{align}
        \sum_{i>i^*} \hat{a}_i p_i+\sum_{i \neq r, i^*} \hat{a}_i p_i+ \hat{a}_r p_r \geq -10\varepsilon. \label{eq:uc:case3:4:supp}
    \end{align}
    
    Combining \eqref{eq:uc:case3:4:supp} and \eqref{eq:uc:case3:5:supp}, we conclude Case 3 of the proof.
\end{proof}

\begin{algorithm}[t]
\caption{A randomized algorithm for monotone $k$-submodular maximization without constraints (adapted from \citep{Iwata2015ImprovedAA})}
\label{alg:uc-monotone}
\begin{algorithmic}[1]
    \STATE {\bfseries Input:} ground set $V$, value oracle access for a monotone $k$-submodular function $f: (k + 1)^V \rightarrow \mathbb{R}_+$.
    \STATE {\bfseries Output:} a vector $\bs \in (k + 1)^V$.
    \STATE Initialize $\bs \leftarrow \mathbf{0}_n$, $t \leftarrow k-1$.
    \FOR{$e \in V$} 
        \STATE $y_i\leftarrow \Delta_{e,i}f(\bs)$ for $i\in [k]$.
        \STATE $\beta \leftarrow \sum_{i=1}^k [y_i]_+^t$.
        \STATE Initialize $p \gets \mathbf{0}_k$.
        \IF{$\beta\neq 0$}
            \STATE $p_i \leftarrow \frac{y_i^t}{\beta}$ for $i\in [k]$.
        \ELSE
            \STATE $p_i\leftarrow 
            \begin{cases}
                1 & \text{if } i=1 \\
                0 & \text{otherwise.}
            \end{cases}$ 
            for $i\in [k]$.
        \ENDIF
        \STATE Choose $\bs(e) \sim \mathbb{P}(\bs(e)=i)=p_i$ for all $i\in[k]$.
    \ENDFOR
    \STATE {\bfseries Return} $\bs$.
\end{algorithmic}
\end{algorithm}

\subsection{Proof of Proposition~\ref{prop:uc:monotone:robust}} \label{sec:prf:uc:monotone}
We use the same setup and constructions as in \cref{sec:prf:uc} and further, since we are considering the monotone case, we have $a_i^{(j)} \geq 0$ and $y_i^{(j)} \geq 0$ for all $i\in [k]$ and $j\in [n]$. Thus, for the surrogate function $\hat{f}$, we have
\begin{align}
    \hat{a}_i^{(j)} \geq -2\varepsilon \text{ and } \hat{y}_i^{(j)} \geq -2\varepsilon. \label{eq:uc:monotone}
\end{align}
Now start to prove Proposition~\ref{prop:uc:monotone:robust}. Again, we shall omit the superscript $(j)$ if it is clear from the context. 

\begin{proof}[Proof of Proposition~\ref{prop:uc:monotone:robust}] 
It is trivial to see the number of value oracle queries required is $nk$, as it iterates through all item-type pairs. The following proof will be dedicated to showing the $\alpha$ and $\delta$ terms in Definition~\ref{def:robust}.

By Definition~\ref{def:robust}, we aim to show that the output $\bs$ of Algorithm~\ref{alg:uc-monotone} satisfies $\mathbb{E}[f(\bs)] \geq \frac{k}{2k-1}f(\bo)-(16-\frac{2}{k})n\varepsilon$ when the function evaluation is based on surrogate function $\hat{f}$. Note that in this setting, Lemma~\ref{lem:uc:lem} still holds. Thus, our goal is to show 
\begin{align}
    \sum_{i=1}^k (\hat{a}_{i^*}-\hat{a}_i) p_i \leq (1-\frac{1}{k}) \sum_{i=1}^k \hat{y}_i p_i + 10\varepsilon, \label{eq:uc:monotone:goal}
\end{align}
which is equivalent to \eqref{eq:uc:condition:main} in main text with $c = 1-\frac{1}{k}$ and $d=10$. Recall that $\hat{y}_i+\hat{y}_{i'} \geq -4\varepsilon$ and $\hat{a}_i+\hat{a}_{i'} \geq -4\varepsilon$ for $i, i' \in[k]$ with $i \neq i'$, and $\hat{y}_i \geq \hat{a}_i - 4\varepsilon$ for $i \in[k]$ (c.f. \eqref{eq:uc:pair1},\eqref{eq:uc:pair2} and \eqref{eq:uc:orthant}). We break up the analysis into two cases: $\beta=0$ and $\beta > 0$. 
\paragraph{Case 1:} When $\beta=0$, we have $\hat{y}_i \leq 0$ for all $i\in [k]$ due to the definition of $\beta$. In this case, \eqref{eq:uc:monotone:goal} specializes to  
\begin{align}
    \hat{a}_{i^*}-\hat{a}_1 \leq (1-\frac{1}{k})\hat{y}_1 + 10 \varepsilon.
\end{align}
We have
\begin{align}
    \hat{a}_{i^*}-\hat{a}_1 & \leq \hat{a}_{i^*} +2\varepsilon \tag{using $\hat{a}_1 \geq -2\varepsilon$} \\
    &\leq 6\varepsilon \tag{using $\hat{a}_{i^*} -4\varepsilon \leq \hat{y}_{i^*} \leq 0$} \\
    &\leq -\hat{y}_{i^*} + 6\varepsilon \tag{using $\hat{y}_{i^*} \leq 0$} \\
    &\leq \hat{y}_1 + 10\varepsilon \tag{using $\hat{y}_{i^*} + \hat{y}_1 \geq -4\varepsilon$}\\
    &\leq (1-\frac{1}{k})\hat{y}_1 + 10\varepsilon, \tag{using $\hat{y}_1\leq 0$}
\end{align}
showing the statement.
\paragraph{Case 2:} When $\beta > 0$, we have $\hat{y}_i > 0$ for some $i\in [k]$. In this case, \eqref{eq:uc:monotone:goal} specializes to 
\begin{align}
    \sum_{i=1}^k (\hat{a}_{i^*}-\hat{a}_i) \hat{y}_i^t \leq (1-\frac{1}{k}) \sum_{i=1}^k \hat{y}_i^{t+1} + 10\beta \varepsilon. \label{eq:uc:monotone:case2}
\end{align}
\paragraph{Sub-case 2.a:} If $k=1$, the LHS of \eqref{eq:uc:monotone:case2} is 0 since $i^*=1$. RHS is positive since both $\hat{y}_i$ and $\beta$ are positive. Thus, \eqref{eq:uc:monotone:case2} clearly holds. 

\paragraph{Sub-case 2.b:} If $k\geq 2$. Let $\gamma = (k-1)^{\frac{1}{t}} = t^{\frac{1}{t}}$. We have
\begin{align}
    \sum_{i\neq i^*}(\hat{a}_{i^*}-\hat{a}_i) \hat{y}_i^t &\leq \sum_{i\neq i^*}(\hat{a}_{i^*} + 2\varepsilon) \hat{y}_i^t \tag{using $\hat{a}_i \geq -2\varepsilon$}\\
    &\leq \sum_{i\neq i^*}(\hat{y}_{i^*} + 6\varepsilon) \hat{y}_i^t \tag{using $\hat{a}_i - 4\varepsilon \leq \hat{y}_i$}\\
    & = \sum_{i\neq i^*}\hat{y}_{i^*}\hat{y}_i^t + 6\varepsilon \sum_{i\neq i^*} \hat{y}_i^t \nonumber\\
    &\leq \sum_{i\neq i^*}\hat{y}_{i^*}\hat{y}_i^t + 6\varepsilon \sum_{i\neq i^*} [\hat{y}_i]_+^t \nonumber\\
    &= \sum_{i\neq i^*}\hat{y}_{i^*}\hat{y}_i^t + 6\varepsilon (\beta-[\hat{y}_{i^*}]_+^t) \tag{definition of $\beta$}\\
    &\leq \sum_{i\neq i^*}\hat{y}_{i^*}\hat{y}_i^t + 6\varepsilon \beta \nonumber\\
    &= \frac{1}{\gamma}\left(\gamma \hat{y}_{i^*} \sum_{i\neq i^*}\hat{y}_i^t\right)+ 6\varepsilon \beta\label{eq:uc:monotone:case2:noistar}
\end{align}
Using the weighted AM-GM inequality, $a^{\frac{1}{t+1}}b^{\frac{t}{t+1}} \leq \frac{1}{t+1}a+\frac{t}{t+1}b$ for all $a,b\geq 0$, and letting $a= (\gamma \hat{y}_{i^*} )^{t+1}$ and $b=(\sum_{i\neq i^*}\hat{y}_i^t)^{\frac{t+1}{t}}$, we deduce
\begin{align}
    \frac{1}{\gamma}\left(\gamma \hat{y}_{i^*} \sum_{i\neq i^*}\hat{y}_i^t\right) &\leq \frac{1}{\gamma}\left(\frac{1}{t+1}\left(\gamma \hat{y}_{i^*}\right)^{t+1}+\frac{t}{t+1}(\sum_{i \neq i^*} \hat{y}_i^t)^{\frac{t+1}{t}}\right)
\end{align}

From Holder's inequality, $\sum_i a_i \leq(\sum_i a_i^{\frac{t+1}{t}})^{\frac{t}{t+1}}(\sum_i 1^{t+1})^{\frac{1}{t+1}}$ holds for any non-negative $a_i$'s. By setting $a_i=\hat{y}_i^t$, we have
\begin{align}
    \frac{1}{\gamma}\left(\gamma \hat{y}_{i^*} \sum_{i\neq i^*}\hat{y}_i^t\right) & \leq \frac{1}{\gamma}\left(\frac{1}{t+1}\left(\gamma \hat{y}_{i^*}\right)^{t+1}+\frac{t}{t+1}\left(\sum_{i \neq i^*} \hat{y}_i^{t+1}\right) \cdot\left(\sum_{i \neq i^*} 1^{t+1}\right)^{\frac{1}{t}}\right) \nonumber\\
    & =\frac{1}{\gamma}\left(\frac{1}{t+1}\left(\gamma \hat{y}_{i^*}\right)^{t+1}+\frac{t(k-1)^{1 / t}}{t+1} \sum_{i \neq i^*} \hat{y}_i^{t+1}\right) \nonumber\\
    & =\frac{\gamma^t}{t+1} \sum_{i=1}^k \hat{y}_i^{t+1} \nonumber\\
    &=(1-\frac{1}{k}) \sum_{i=1}^k \hat{y}_i^{t+1}.\label{eq:uc:monotone:case2:term1}
\end{align}
Plug \eqref{eq:uc:monotone:case2:term1} back into \eqref{eq:uc:monotone:case2:noistar}, we get
\begin{align}
    \sum_{i\neq i^*}(\hat{a}_{i^*}-\hat{a}_i) \hat{y}_i^t &\leq (1-\frac{1}{k}) \sum_{i=1}^k \hat{y}_i^{t+1} + 6\varepsilon \beta \nonumber\\
    &\leq (1-\frac{1}{k}) \sum_{i=1}^k \hat{y}_i^{t+1} + 10\varepsilon \beta,
\end{align}
which is desired. 

\end{proof}

\begin{algorithm}[t]
\caption{Monotone $k$-submodular maximization with a IS constraint \citep{ohsaka2015monotone}}
\label{alg:is}
\begin{algorithmic}[1]
    \STATE {\bfseries Input:} value oracle access for a monotone $k$-submodular function $f: (k + 1)^V \rightarrow \mathbb{R}_+$, integers $B_1, \ldots, B_k \in \mathbb{Z}_+$.
    \STATE {\bfseries Output:} a vector $\bs$ satisfying $\operatorname{supp}_i(\bs) = B_i$ for each $i\in [k]$.
    \STATE Initialize $\bs \leftarrow \mathbf{0}$ and $B\leftarrow \sum_{i\in [k]} B_i$.
    \FOR{$j \in [B]$} 
        \STATE $I\leftarrow \{i\in [k]|\operatorname{supp}_i(\bs)< B_i\}$.
        \STATE $(e,i)\leftarrow \argmax_{e\in V\setminus \operatorname{supp}(\bs), i\in I} \Delta_{e,i}f(\bs)$.
        \STATE $\bs(e)\leftarrow i$.
    \ENDFOR
    \STATE {\bfseries Return} $\bs$.
\end{algorithmic}
\end{algorithm}

\subsection{Proof of Proposition~\ref{prop:is:robust}} \label{sec:prf:is}
\begin{proof}
    Let $\left(e^{(j)}, i^{(j)}\right) \in V \times[k]$ be the pair greedily chosen in this iteration using the surrogate function $\hat{f}$, and let $\bs^{(j)}$ be the solution after this iteration. Let $\bo$ be an optimal solution. For each $j \in[B]$, we define $S_i^{(j)}=\operatorname{supp}_i(\bo^{(j-1)}) \backslash \operatorname{supp}_i(\bs^{(j-1)})$. Following the construction of \citet{ohsaka2015monotone}, we iteratively define $\bo^{(0)}=\boldsymbol{o}, \bo^{(1)}, \ldots, \bo^{(B)}$ as follows.
    We have two cases to consider:

    \textbf{Case 1:} Suppose that $e^{(j)} \in S_{i'}^{(j)}$ for some $i' \neq i^{(j)}$. In this case, let $o^{(j)}$ be an arbitrary element in $S_{i^{(j)}}^{(j)}$. Then, we define $\boldsymbol{o}^{(j-1 / 2)}$ as the resulting vector obtained from $\boldsymbol{o}^{(j-1)}$ by assigning 0 to the $e^{(j)}$-th element and the $o^{(j)}$-th element, and then define $\boldsymbol{o}^{(j)}$ as the resulting vector obtained from $\boldsymbol{o}^{(j-1 / 2)}$ by assigning $i^{(j)}$ to the $e^{(j)}$-th element and $i'$ to the $o^{(j)}$-th element.
    
    \textbf{Case 2:} Suppose that $e^{(j)} \notin S_{i'}^{(j)}$ for any $i' \neq i^{(j)}$. In this case, we set $o^{(j)}=e^{(j)}$ if $e^{(j)} \in S_{i^{(j)}}^{(j)}$, and we set $o^{(j)}$ to be an arbitrary element in $S_{i^{(j)}}^{(j)}$ otherwise. Then, we define $\boldsymbol{o}^{(j-1 / 2)}$ as the resulting vector obtained from $\boldsymbol{o}^{(j-1)}$ by assigning 0 to the $o^{(j)}$-th element, and then define $\boldsymbol{o}^{(j)}$ as the resulting vector obtained from $\boldsymbol{o}^{(j-1 / 2)}$ by assigning $i^{(j)}$ to the $e^{(j)}$-th element.

    Intuitively, we want a transition from $\bo$ to $\bs$ by swapping in item-type pairs in $\bs$ one by one to $\bo$, in the order of the pair is chosen by the greedy algorithm. In case 1, if the considered item $e^{(j)}$ is already in $\bo^{(j-1)}$ but with another type $i' \neq i^{(j)}$, we obtain $\bo^{(j-1/2)}$ from $\bo^{(j-1)}$ by assign type 0 to item $e^{(j)}$, and obtain $\bo^{(j)}$ from $\bo^{(j-1/2)}$ by assign type $i^{(j)}$ to item $e^{(j)}$. In case 2, if there is no type $i' \neq i^{(j)}$ that is assigned to item $e^{(j)}$ in $\bo^{(j-1)}$, there are two sub-cases. In the first sub-case, if $(e^{(j)}, i^{(j)})$ is already in $\bo^{(j-1)}$, we set $\bo^{(j)}=\bo^{(j-1)}$ and obtain $\bo^{(j-1/2)}$ from $\bo^{(j-1)}$ by assign type 0 to item $e^{(j)}$. In the second sub-case, if $e^{(j)}$ is not assigned any type in $\bo^{(j-1)}$, we obtain $\bo^{(j-1/2)}$ from $\bo^{(j-1)}$ by assign type 0 to an arbitrary item with type $i^{(j)}$ in $\bo^{(j-1)}$, and obtain $\bo^{(j)}$ from $\bo^{(j-1/2)}$ by assign type $i^{(j)}$ to item $e^{(j)}$. Note that $\left|\operatorname{supp}_i(\bo^{(j)})\right|=B_i$ holds for every $i \in[k]$ and $j \in\{0,1, \ldots, B\}$, and $\bo^{(B)}=\bs^{(B)}=\bs$. Moreover, we have $\bs^{(j-1)} \preceq \bo^{(j-1 / 2)}$ for every $j \in[B]$. We further denote $\Delta_{e,i}\hat{f}(\bx) = \hat{f}(X_1, \ldots, X_{i-1}, X_i \cup\{e\}, X_{i+1}, \ldots, X_k)-\hat{f}(X_1, \ldots, X_k)$ as analogous to $\Delta_{e,i}f(\bx)$ under surrogate function $\hat{f}$.

    We consider $f(\bo^{(j-1)})-f(\bo^{(j)})$ in these two cases.

    \textbf{Case 1:} In this case, $e^{(j)} \in S_{i'}^{(j)}$ for some $i' \neq i^{(j)}$. By construction, we have $$f(\bo^{(j-1)})-f(\bo^{(j-1/2)})= \Delta_{o^{(j)}, i^{(j)}} f\left(\bo^{(j-1 / 2)}\right)+ \Delta_{e^{(j)}, i'} f\left(\bo^{(j-1 / 2)}\right)$$ and $$f(\bo^{(j)})-f(\bo^{(j-1/2)})= \Delta_{e^{(j)}, i^{(j)}} f\left(\bo^{(j-1 / 2)}\right)+ \Delta_{o^{(j)}, i'} f\left(\bo^{(j-1 / 2)}\right).$$
    Thus,
    \begin{align}
        f(\bo^{(j-1)})-f(\bo^{(j)}) &= \left(f(\bo^{(j-1)})-f(\bo^{(j-1/2)})\right) - \left(f(\bo^{(j)})-f(\bo^{(j-1/2)})\right) \nonumber\\
        &= \Delta_{o^{(j)}, i^{(j)}} f\left(\bo^{(j-1 / 2)}\right)+ \Delta_{e^{(j)}, i'} f(\bo^{(j-1 / 2)}) -  \Delta_{e^{(j)}, i^{(j)}} f(\bo^{(j-1 / 2)}) \nonumber\\
        &\qquad -\Delta_{o^{(j)}, i'} f(\bo^{(j-1 / 2)}) \tag{by construction} \\
        & \leq \Delta_{o^{(j)}, i^{(j)}} f(\bo^{(j-1 / 2)})+ \Delta_{e^{(j)}, i'} f(\bo^{(j-1 / 2)}) \tag{monotonicity} \\
        &\leq \Delta_{o^{(j)}, i^{(j)}} f(\bs^{(j-1)})+ \Delta_{e^{(j)}, i'} f(\bs^{(j-1)}) \tag{$\bs^{(j-1)} \preceq \bo^{(j-1 / 2)}$ and orthant submodularity}\\
        &\leq \Delta_{o^{(j)}, i^{(j)}} \hat{f}(\bs^{(j-1)})+ \Delta_{e^{(j)}, i'} \hat{f}(\bs^{(j-1)}) + 4\varepsilon \\
        &\leq 2\Delta_{e^{(j)}, i^{(j)}} \hat{f}(\bs^{(j-1)}) + 4\varepsilon\\ 
        &= 2(\hat{f}(\bs^{(j)})-\hat{f}(\bs^{(j-1)})) + 4\varepsilon. \label{eq:is2:case1}
    \end{align}
    where the last inequality follows from greedy rule and that both pair $\left(o^{(j)}, i^{(j)}\right)$ and $\left(e^{(j)}, i'\right)$ are available after we selected $\bs^{(j)}$, which we verify as follows: since $o^{(j)}\in S_{i^{(j)}}^{(j)}$, $o^{(j)}$ is still available; $i^{(j)}$ is the type of the next greedy pair, thus available; $e^{(j)}$ is the item of the next greedy pair, thus available; $e^{(j)} \in S_{i'}^{(j)}$ indicates $i'$ is still available. 

    \textbf{Case 2:} In this case, $e^{(j)} \notin S_{i'}^{(j)}$ for any $i' \neq i^{(j)}$. 
    \begin{align}
        f(\bo^{(j-1)})-f(\bo^{(j)}) &= \Delta_{o^{(j)}, i^{(j)}} f(\bo^{(j-1 / 2)}) -  \Delta_{e^{(j)}, i^{(j)}} f(\bo^{(j-1 / 2)})\tag{by construction} \\
        & \leq \Delta_{o^{(j)}, i^{(j)}} f(\bo^{(j-1 / 2)}) \tag{monotonicity} \\
        &\leq \Delta_{o^{(j)}, i^{(j)}} f(\bs^{(j-1)}) \tag{$\bs^{(j-1)} \preceq \bo^{(j-1 / 2)}$ and orthant submodularity}\\
        &\leq \Delta_{o^{(j)}, i^{(j)}} \hat{f}(\bs^{(j-1)})+2\varepsilon \\
        &\leq \Delta_{e^{(j)}, i^{(j)}} \hat{f}(\bs^{(j-1)}) + 2\varepsilon \tag{greedy rule and the pair $\left(o^{(j)}, i^{(j)}\right)$ is available}\\
        &\leq 2(\hat{f}(\bs^{(j)})-\hat{f}(\bs^{(j-1)})) + 4\varepsilon. \label{eq:is2:case2}
    \end{align}

    Combining \eqref{eq:is2:case1} and \eqref{eq:is2:case2} we have in both cases,
    \begin{align}
        f(\bo^{(j-1)})-f(\bo^{(j)})\leq 2(\hat{f}(\bs^{(j)})-\hat{f}(\bs^{(j-1)})) + 4\varepsilon. \label{eq:is2:diff}
    \end{align}
    
    Finally, 
    \begin{align}
        f(\bo)-f(\bs) &=\sum_{j=1}^B\left(f(\bo^{(j-1)})-f(\bo^{(j)})\right)\\
        &\leq 2\sum_{j=1}^B\left(\hat{f}(\bs^{(j)})-\hat{f}(\bs^{(j-1)})\right) + 4B\varepsilon \tag{using \eqref{eq:is2:diff}}\\
        &=2(\hat{f}(\bs)-\hat{f}(\mathbf{0})) + 4B\varepsilon \nonumber\\
        &\leq 2(f(\bs)-f(\mathbf{0})) + 4(B+1)\varepsilon \nonumber\\
        &\leq 2f(\bs)+ 4(B+1)\varepsilon \nonumber,
    \end{align}

    Rearranging, we get $f(\bs)\geq \frac{1}{3}f(\bo) - \frac{4}{3}(B+1)\varepsilon$.
    
\end{proof}

\begin{algorithm}[t]
\caption{$k$-submodular maximization with a matroid constraint \citep{sakaue2017maximizing}}
\label{alg:matroid}
\begin{algorithmic}[1]
    \STATE {\bfseries Input:} value oracle access for a $k$-submodular function $f: (k + 1)^V \rightarrow \mathbb{R}_+$, a matroid $(V,\mathcal{F})$,  given by the evaluation and independence oracle respectively.
    \STATE {\bfseries Output:} a vector $\bs$ satisfying $\operatorname{supp}(\bs) \in \mathcal{B}$.
    \STATE Initialize $\bs \leftarrow \mathbf{0}$.
    \FOR{$j \in [M]$} 
        \STATE Construct $E(\bs)$ using the independence oracle.
        \STATE $(e,i)\leftarrow \argmax_{e\in E(\bs), i\in [k]} \Delta_{e,i}f(\bs)$.
        \STATE $\bs(e)\leftarrow i$.
    \ENDFOR
    \STATE {\bfseries Return} $\bs$.
\end{algorithmic}
\end{algorithm}

\subsection{Proof of Proposition~\ref{prop:matroid:robust}} \label{sec:prf:matroid}

Denote $\mathcal{B}$ as the set of all bases associated with a matroid $(E,\mathcal{F})$. We first state two important lemmas that will be used in the proof. We refer to \citet{sakaue2017maximizing} for the detailed proofs.
\begin{lemma}\label{lem:matroid:lem1}\cite{sakaue2017maximizing}
    The size of any maximal optimal solution for maximizing a monotone $k$-submodular function under a matroid constraint is $M$.
\end{lemma}

\begin{lemma}\label{lem:matroid:lem2}\cite{sakaue2017maximizing}
    Suppose $A\in \mathcal{F}$ and $B\in \mathcal{B}$ satisfy $A \subsetneq B$. Then, for any $e\notin A$ satisfying $A \cup \{e\} \in \mathcal{F}$, there exists $e'\in B\setminus A$ such that $\{B\setminus \{e'\}\}\cup \{e\} \in \mathcal{B}$.
\end{lemma}
Now we prove Proposition~\ref{prop:matroid:robust}.

\begin{proof}
    Let $\left(e^{(j)}, i^{(j)}\right)$ be the pair chosen greedily at the $j$ th iteration using the surrogate function $\hat{f}$, and $\bs^{(j)}$ be the solution after the $j$th iteration. Let $s^{(0)}=\mathbf{0}$ and $s=s^{(M)}$, the output of Algorithm~\ref{alg:matroid}. Define $S^{(j)}:=\operatorname{supp}(\bs^{(j)})$ for each $j \in[M]$. Let $\bo$ be a maximal optimal solution and let $O^{(j)}:=\operatorname{supp}(\bo^{(j)})$ for each $j \in[M]$. Now we will construct a sequence of vectors $\bo^{(0)}=\bo, \bo^{(1)}, \ldots, \bo^{(M-1)}, \bo^{(M)}=\bs$ satisfying the following:
    \begin{align}
        \boldsymbol{s}^{(j)} \prec \boldsymbol{o}^{(j)} \quad \text { if } &j=0,1, \ldots, M-1, \quad \text { and } \quad \boldsymbol{s}^{(j)}=\boldsymbol{o}^{(j)}=\boldsymbol{s} \quad \text { if } j=M . \label{eq:monomatroid:cond1}\\
        O^{(j)} \in \mathcal{B} \quad \text { for } &j=0,1, \ldots, M. \label{eq:monomatroid:cond2}
    \end{align}
    
    More specifically, we see how to obtain $\bo^{(j)}$ from $\bo^{(j-1)}$ satisfying \eqref{eq:monomatroid:cond1} and \eqref{eq:monomatroid:cond2}. Note that $\bs^{(0)}=\mathbf{0}$ and $\bo^{(0)}=\bo$ satisfy \eqref{eq:monomatroid:cond1} and \eqref{eq:monomatroid:cond2}. We now describe how to obtain $\bo^{(j)}$ from $\bo^{(j-1)}$, assuming that $\bo^{(j-1)}$ satisfies
    $$
    \bs^{(j-1)} \prec \bo^{(j-1)} \text {, and } O^{(j-1)} \in \mathcal{B}.
    $$
    Since $\boldsymbol{s}^{(j-1)} \prec \boldsymbol{o}^{(j-1)}$ means $S^{(j-1)} \subsetneq O^{(j-1)}$, and $e^{(j)}$ is chosen to satisfy $S^{(j-1)} \cup\left\{e^{(j)}\right\} \in \mathcal{F}$, we see from Lemma~\ref{lem:matroid:lem2} that there exists $e' \in O^{(j-1)} \setminus S^{(j-1)}$ satisfying $\{O^{(j-1)} \setminus\{e'\}\} \cup \{e^{(j)}\} \in \mathcal{B}$. 
    We let $o^{(j)}=e'$ and define $\bo^{(j-1 / 2)}$ as the vector obtained by assigning 0 to the $o^{(j)}$ th element of $\boldsymbol{o}^{(j-1)}$. 
    We then define $\boldsymbol{o}^{(j)}$ as the vector obtained from $\boldsymbol{o}^{(j-1 / 2)}$ by assigning type $i^{(j)}$ to element $e^{(j)}$. 
    The vector thus constructed, $\bo^{(j)}$, satisfies
    $$O^{(j)}=\{O^{(j-1)} \setminus \{o^{(j)}\}\} \cup \{e^{(j)}\} \in \mathcal{B}.$$
    Furthermore, since $\boldsymbol{o}^{(j-1 / 2)}$ satisfies
    $$
    \bs^{(j-1)} \preceq \bo^{(j-1 / 2)},
    $$
    we have the following property for $\bo^{(j)}$ :
    $$
    \bs^{(j)} \prec \bo^{(j)} \quad \text { if } j=1, \ldots, M-1, \quad \text { and } \quad \bs^{(j)}=\bo^{(j)}=\bs \text { if } j=M,
    $$
    where the strictness of the inclusion for $j \in[M-1]$ can be easily confirmed from $|S^{(j)}|=j<M=|O^{(j)}|$. Thus, applying the above discussion for $j=1, \ldots, M$ iteratively, we see the obtained sequence of vectors $\bo^{(0)}, \bo^{(1)}, \ldots, \bo^{(M)}$ satisfies \eqref{eq:monomatroid:cond1} and \eqref{eq:monomatroid:cond2}.
    We now prove the following inequality for $j \in[M]$ :
    \begin{align}
        \hat{f}(\bs^{(j)})-\hat{f}(\bs^{(j-1)}) \geq f(\bo^{(j-1)})-f(\bo^{(j)}) -2\varepsilon. \label{eq:mono:matroid:diff}
    \end{align}
    From $S^{(j-1)} \subsetneq O^{(j-1)}$ and $o^{(j)} \in O^{(j-1)} \backslash S^{(j-1)}$, we get $S^{(j-1)} \cup\left\{o^{(j)}\right\} \subseteq O^{(j-1)} \in \mathcal{B}$ for each $j \in[M]$. Thus we obtain the following inclusion from (M2) for each $j \in[M]$ :
    $$
    S^{(j-1)} \cup \{o^{(j)}\} \in \mathcal{F}.
    $$
    Hence, from $o^{(j)} \notin S^{(j-1)}$, we get $o^{(j)} \in E\left(\bs^{(j-1)}\right)$, where $E(\bs)$ is constructed by the independence oracle and contains the available elements given current solution $\bs$. Therefore, for the pair $\left(e^{(j)}, i^{(j)}\right)$, which is chosen greedily, we have
    \begin{align}
        \Delta_{e^{(j)}, i^{(j)}} \hat{f}(\bs^{(j-1)}) \geq \Delta_{o^{(j)}, \bo^{(j-1)}(o^{(j)})} \hat{f}(\bs^{(j-1)}) . \label{eq:mono:matroid:greedy}
    \end{align}
    Furthermore, since $\boldsymbol{s}^{(j-1)} \preceq \boldsymbol{o}^{(j-1 / 2)}$ holds, orthant submodularity implies
    \begin{align}
        \Delta_{o^{(j)}, \bo^{(j-1)}(o^{(j)})} f(\bs^{(j-1)}) \geq \Delta_{o^{(j)}, \bo^{(j-1)}(o^{(j)})} f(\bo^{(j-1 / 2)}).\label{eq:mono:matroid:submod}
    \end{align}
    Using \eqref{eq:mono:matroid:greedy} and \eqref{eq:mono:matroid:submod}, we have:
    \begin{align}
        f(\bo^{(j-1)})-f(\bo^{(j)}) &= \Delta_{o^{(j)}, \bo^{(j-1)}(o^{(j)})} f(\bo^{(j-1 / 2)})-\Delta_{e^{(j)}, i^{(j)}} f(\bo^{(j-1 / 2)}) \tag{by construction}\\
        &\leq \Delta_{o^{(j)}, \bo^{(j-1)}(o^{(j)})} f(\bo^{(j-1 / 2)}) \tag{monotonicity} \\
        &\leq \Delta_{o^{(j)}, \bo^{(j-1)}(o^{(j)})} f(\bs^{(j-1)}) \tag{using \eqref{eq:mono:matroid:submod}}\\
        &\leq \Delta_{o^{(j)}, \bo^{(j-1)}(o^{(j)})} \hat{f}(\bs^{(j-1)}) + 2\varepsilon \nonumber\\
        &\leq \Delta_{e^{(j)}, i^{(j)}} \hat{f}(\bs^{(j-1)}) + 2\varepsilon \tag{using \eqref{eq:mono:matroid:greedy}}\\
        &= \hat{f}(\bs^{(j)})-\hat{f}(\bs^{(j-1)}) + 2\varepsilon,
    \end{align}
    and we get \eqref{eq:mono:matroid:diff}. Finally, 
    \begin{align}
        f(\bo)-f(\bs) &=\sum_{j=1}^M\left(f(\bo^{(j-1)})-f(\bo^{(j)})\right)\\
        &\leq \sum_{j=1}^M\left(\hat{f}(\bs^{(j)})-\hat{f}(\bs^{(j-1)})\right) + 2M\varepsilon \tag{using \eqref{eq:mono:matroid:diff}}\\
        &=(\hat{f}(\bs)-\hat{f}(\mathbf{0})) + 2M\varepsilon \nonumber\\
        &\leq (f(\bs)-f(\mathbf{0})) + 2(M+1)\varepsilon \nonumber\\
        &\leq f(\bs)+ 2(M+1)\varepsilon \nonumber,
    \end{align}
    Rearranging, we get $f(\bs)\geq \frac{1}{2}f(\bo) - (M+1)\varepsilon$.
    
\end{proof}

\subsection{Proof of Proposition~\ref{prop:matroid:robust2}} \label{sec:prf:matroid2}
The following lemma guarantees that there exists an optimal solution with all items included (so we just need to decide on the types).

\begin{lemma}\label{lem:nmksm:lem1}\cite{sun2022maximize}
    The size of any maximal optimal solution for maximizing a non-monotone $k$-submodular function under a matroid constraint is still $M$.
\end{lemma}

Now we prove Proposition~\ref{prop:matroid:robust2}.

\begin{proof}
    We construct the sequence $\bo^{(j)}$, $\bo^{(j-1/2)}$ the same way as in the proof for Proposition~\ref{prop:matroid:robust}. By the same construction, we have that 
    \begin{align}
        \bs^{(j)} \prec \bo^{(j)} \quad \text { if } &j=0,1, \ldots, M-1, \quad \text { and } \quad \bs^{(j)}=\bo^{(j)}=\bs \quad \text { if } j=M . \label{eq:nmksm:cond1}\\
        O^{(j)} \in \mathcal{B} \quad \text { for } &j=0,1, \ldots, M. \label{eq:nmksm:cond2}
    \end{align}
    and we still have \eqref{eq:mono:matroid:greedy} and \eqref{eq:mono:matroid:submod}. Since we do not have monotonicity anymore, we will exploit the pairwise submodularity property. Besides the chosen type $i^{(j)}$ for the item $e^{(j)}$ at each iteration of Algorithm~\ref{alg:matroid}, we consider another type $h^{(j)} \neq i^{(j)}$. From pariwise monotonicity we have
    \begin{align}
        \Delta_{e^{(j)},i^{(j)}}f(\bo^{(j-1/2)}) + \Delta_{e^{(j)},h^{(j)}}f(\bo^{(j-1/2)}) \geq 0. \label{eq:nmksm:pair}
    \end{align}
    We perform similar calculation as in the proof for Proposition~\ref{prop:matroid:robust}:
    \begin{align}
        \hat{f}(\bs^{(j)})-\hat{f}(\bs^{(j-1)}) &= \Delta_{e^{(j)}, i^{(j)}} \hat{f}(\bs^{(j-1)}) \nonumber\\
        &\geq \Delta_{o^{(j)}, \bo^{(j-1)}(o^{(j)})} \hat{f}(\bs^{(j-1)}) \tag{using \eqref{eq:mono:matroid:greedy}}\\
        &\geq \Delta_{o^{(j)}, \bo^{(j-1)}(o^{(j)})} f(\bs^{(j-1)}) - 2\varepsilon \nonumber\\
        &\geq \Delta_{o^{(j)}, \bo^{(j-1)}(o^{(j)})} f(\bo^{(j-1/2)}) - 2\varepsilon \tag{using \eqref{eq:mono:matroid:submod}}\\
        &\geq \Delta_{o^{(j)}, \bo^{(j-1)}(o^{(j)})} f(\bo^{(j-1/2)}) - \Delta_{e^{(j)},i^{(j)}}f(\bo^{(j-1/2)}) - \Delta_{e^{(j)},h^{(j)}}f(\bo^{(j-1/2)}) - 2\varepsilon \tag{using \eqref{eq:nmksm:pair}} \\
        &\geq \Delta_{o^{(j)}, \bo^{(j-1)}(o^{(j)})} f(\bo^{(j-1/2)}) - \Delta_{e^{(j)},i^{(j)}}f(\bo^{(j-1/2)}) - \Delta_{e^{(j)},h^{(j)}}f(\bs^{(j-1)}) - 2\varepsilon \tag{$\bs^{(j-1)} \preceq \bo^{(j-1 / 2)}$ as shown in previous section}\\
        &\geq \Delta_{o^{(j)}, \bo^{(j-1)}(o^{(j)})} f(\bo^{(j-1/2)}) - \Delta_{e^{(j)},i^{(j)}}f(\bo^{(j-1/2)}) - \Delta_{e^{(j)},i^{(j)}}f(\bs^{(j-1)}) - 2\varepsilon \tag{greedy rule}\\
        &\geq \Delta_{o^{(j)}, \bo^{(j-1)}(o^{(j)})} f(\bo^{(j-1/2)}) - \Delta_{e^{(j)},i^{(j)}}f(\bo^{(j-1/2)}) - \Delta_{e^{(j)},i^{(j)}}\hat{f}(\bs^{(j-1)}) - 4\varepsilon \nonumber\\
        &= f(\bo^{(j-1)})-f(\bo^{(j-1/2)}) - f(\bo^{(j)}) + f(\bo^{(j-1/2)}) - (\hat{f}(\bs^{(j)})-\hat{f}(\bs^{(j-1)})) - 4\varepsilon \nonumber\\
        &= f(\bo^{(j-1)}) - f(\bo^{(j)}) - (\hat{f}(\bs^{(j)})-\hat{f}(\bs^{(j-1)})) - 4\varepsilon. \nonumber
    \end{align}
    Adding both sides by $\hat{f}(\bs^{(j)})-\hat{f}(\bs^{(j-1)})$ we get
    \begin{align}
        2(\hat{f}(\bs^{(j)})-\hat{f}(\bs^{(j-1)})) \geq f(\bo^{(j-1)})-f(\bo^{(j)})- 4\varepsilon. \label{eq:nmksm:diff}
    \end{align}
    Finally, 
    \begin{align}
        f(\bo)-f(\bs) &=\sum_{j=1}^M\left(f(\bo^{(j-1)})-f(\bo^{(j)})\right)\\
        &\leq 2\sum_{j=1}^M\left(\hat{f}(\bs^{(j)})-\hat{f}(\bs^{(j-1)})\right) + 4B\varepsilon \tag{using \eqref{eq:nmksm:diff}}\\
        &=2(\hat{f}(\bs)-\hat{f}(\mathbf{0})) + 4M\varepsilon \nonumber\\
        &\leq 2(f(\bs)-f(\mathbf{0})) + 4(M+1)\varepsilon \nonumber\\
        &\leq 2f(\bs)+ 4(M+1)\varepsilon \nonumber,
    \end{align}
    Rearranging, we get $f(\bs)\geq \frac{1}{3}f(\bo) - \frac{4}{3}(M+1)\varepsilon$.
\end{proof}

\section{Missing Offline Algorithms}
We present the pseudo-codes that is missing in the main sections. Algorithm~\ref{alg:uc-monotone} is adapted from \citet{Iwata2015ImprovedAA} for the problem of maximizing an unconstrained monotone $k$-submobular function; Algorithm~\ref{alg:is} is proposed in \citet{ohsaka2015monotone} for the problem of maximizing a monotone $k$-submobular function under IS constraints; and Algorithm~\ref{alg:matroid} is proposed in \citet{sakaue2017maximizing} for the problem of maximizing a monotone $k$-submobular function under a matroid constraint. Later \citet{sun2022maximize} showed that Algorithm~\ref{alg:matroid} can also handle non-monotone objective functions. One thing to be noticed is that in Line 4, \cref{alg:matroid} requires the value of $M$, the rank of the matroid. However, in practice, we need not calculate the value of $M$ beforehand. Instead, we continue the iteration while $E(\bs)$ is nonempty, which we check in Line 5. We can confirm that this modification does not change the output as follows. As long as $|\operatorname{supp}(\bs)| < M$, exactly one element is added to $\operatorname{supp}(\bs)$ at each iteration due to (M3), and, if $|\operatorname{supp}(\bs)| = M$, the iteration stops since $\operatorname{supp}(\bs)$ is a maximal independent set.

\section{More Related Works} \label{sec:supp:rw}

\paragraph{Offline $k$-submodular function maximization:} 
$k$-submodular functions were first introduced by \citet{Huber2012TowardsMK} as a generalization of \textit{bisubmodular functions}, which correspond to $k=2$. For unconstrained $k$-submodular maximization, \citet{Iwata2015ImprovedAA} showed that even achieving an approximation ratio $\alpha\in (\frac{k+1}{2k}, 1]$ is NP-hard. If the objective is monotone, there exists a deterministic algorithm achieving $1/2$ approximation \citep{Ward2014MaximizingKF} and a randomized algorithm achieving $\frac{k}{2k-1}$ approximation guarantee \citep{Iwata2015ImprovedAA}. When the objective is non-monotone, \citet{Ward2014MaximizingKF} achieved $\max\{1/3, 1/(1 + a)\}$ approximation guarantee using $\mathcal{O}(kn)$ number of function evaluations, where $a=\max\{1,\sqrt{(k-1)/4}\}$. It has been further improved to $1/2$ \citep{Iwata2015ImprovedAA} and $\frac{k^2+1}{2k^2+1}$ \citep{oshima2021improved}.
 
For monotone size constraints, two types of constraints have been considered in the literature, namely total size (TS) constraints, where the number of items selected shares a common budget, and individual size (IS) constraints, where each of the $k$ types has a budget. 
\citet{ohsaka2015monotone} analyzed the greedy algorithm and obtained $1/2$ and $1/3$ approximation guarantees for total size (TS) constraints and individual size (IS) constraints, respectively. 
\citet{nie23size} proposed threshold greedy algorithms for monotone TS and IS with improved query complexity. 

For maximizing a monotone $k$-submodular function under a matroid constraint, \citet{sakaue2017maximizing} showed the greedy algorithm can achieve $1/2$ approximation, and \citet{matsuoka2021maximization} proposed an algorithm that achieves $\frac{1}{1+c}$ approximation, where $c$ is the curvature. When the objective function is non-monotone, a $1/3$ approximation algorithm is presented in \citet{sun2022maximize}. When the curvature $c$ is known, \citet{matsuoka2021maximization} proposed an algorithm that achieves a $\frac{1}{1+c}$-approximation for matroid constraints and $\frac{1}{1+2c}$-approximation for individual size constraints.

For other constraints, \citet{Tang2021OnMaximizing,Chen2022Monotone} proposed algorithms inspired by \citet{khuller1999budgeted,Sviridenko2004ANO} for knapsack constraints. \citet{xiao2023approximation} considered knapsack constraints under both monotone or non-monotone cases. \citet{yu2023onmax} considered intersection of knapsack and matroid constraints. \citet{pham2021streaming} proposed an algorithm for the streaming setting under knapsack constraints. While the term ``streaming" often implies an online context, it is important to note the setup is distinct from CMAB problems.  In streaming, an exact value oracle is available, the action space is restricted to (subsets of) the elements revealed so far, and the evaluation is after the stream complete. Here, the term ``online" specifically refers to the arrival of data has some ordering.

\section{Applications} \label{sec:supp:app}

In this section, we give some real world appications that motivates our problem of interests.

{\bf Application 1 (Influence maximization with $k$ topics):} Let $G=(V, E)$ be a social network with an edge probability $p_{u, v}^i$ for each edge $(u, v) \in E$, representing the probability of $u$ influencing $v$ on the $i$-th topic. Given a seed set $(S_1,\cdots,S_k)$, the diffusion process of the rumor about the $i$-th topic starts by activating vertices in $S_i$, and propagates independently from other topics. The goal is to maximize the total influence (number of users that is influenced by at least one topic). In the well-studied models of influence propagation — the independent cascades and the linear threshold models — the influence function is monotone and $k$-submodular \citep{ohsaka2015monotone}. In an online version of the problem, the underlying graph structure is not known to the agent. For each time step, the agent selects a seed set, and observes a reward representing the number of influenced users. 

{\bf Application 2 (Sensor placement):} We are given $k$ types of sensors, each of which provides different measurements, and we have $B_i$ sensors of type $i$ for each $i \in [k]$. The ground set $V$ is a set of possible locations for the sensors. The goal is to equip each location with at most one sensor in order to maximize the total information gained from the sensors. This problem can also be modeled as $k$-submodular maximization with an individual size constraint \citep{ohsaka2015monotone}. In an online version of this problem, for each time step, we select locations to place the sensors, then we collect data to calculate entropy, and in the next time step we repeat the same process by re-allocating sensors to different locations.

{\bf Application 3 (Ad allocation):} We have $k$ advertisers that are known in advance and $n$ ad impressions that arrive online one at a time. Each advertiser $i \in[k]$ has a contract of $B_i$ impressions. For each impression $j$ and each advertiser $i$, there is a non-negative weight $w_{j i}$ that captures how much value advertiser $i$ accrues from being allocated impression $j$. When impression $j$ arrives, the values $\left\{w_{j i}: i \in[k]\right\}$ are revealed, and the algorithm needs to allocate impression $j$ to at most one advertiser. Letting $X_i$ denote the set of impressions allocated to advertiser $i$, the total revenue is $\sum_{i=1}^k \max \{\sum_{j \in S_i} w_{j i}: S_i \subseteq X_i,\left|S_i\right| \leq B_i\}$. This problem can be formulated as a special case of $k$-submodular maximization with a partition matroid constraint \citep{ene2022streaming}.

\end{document}